\documentclass{article}

\usepackage{graphicx}
\usepackage{subfigure}
\usepackage{zref}
\usepackage{xcolor}
\usepackage{booktabs} 
\usepackage{amsmath,multirow}
\usepackage{amssymb}
\usepackage{enumerate}
\usepackage{graphicx}
\usepackage{natbib}
\usepackage{subfigure}
\usepackage{tabulary}
\usepackage[colorlinks=true,citecolor=brown,urlcolor=gray]{hyperref}
\usepackage{url}            
\usepackage{amsfonts}       
\usepackage{nicefrac}       
\usepackage{microtype}      
\usepackage{amstext}
\usepackage{amsthm}
\usepackage{bm}
\usepackage{wrapfig}
\usepackage{bbm, comment}
\usepackage{color}
\usepackage{float}
\usepackage{thm-restate}
\usepackage{MnSymbol,wasysym,marvosym,ifsym}
\usepackage{algorithm}
\usepackage{algorithmic}
\usepackage{caption}
\usepackage[utf8]{inputenc}

\usepackage{amsmath,amsfonts,bm,amssymb,mathtools,amsthm}
\usepackage{color,xcolor}
\usepackage{amsmath}
\usepackage{xspace} 
\def\onedot{$\mathsurround0pt\ldotp$}
\def\eg{\emph{e.g}\onedot, }

\def\ie{\emph{i.e}\onedot, }
\newtheorem{theorem}{Theorem}

\newtheorem{assumption}{Assumption}
\newtheorem{proposition}{Proposition}

\newcommand{\diff}{\mathrm{d}}

\def\Figref#1{Fig.~\ref{#1}}

\def\secref#1{section~\ref{#1}}

\def\eqref#1{equation~\ref{#1}}
\def\Eqref#1{Eq.~(\ref{#1})}

\def\Algref#1{Algorithm~\ref{#1}}

\def\1{\bm{1}}

\DeclareMathAlphabet{\mathsfit}{\encodingdefault}{\sfdefault}{m}{sl}
\SetMathAlphabet{\mathsfit}{bold}{\encodingdefault}{\sfdefault}{bx}{n}

\def\gB{{\mathcal{B}}}

\def\gN{{\mathcal{N}}}

\def\gR{{\mathcal{R}}}

\def\gX{{\mathcal{X}}}
\def\gY{{\mathcal{Y}}}
\def\gZ{{\mathcal{Z}}}

\newcommand\numberthis{\addtocounter{equation}{1}\tag{\theequation}}

\newcommand{\E}{\mathbb{E}}
\newcommand{\Ls}{\mathcal{L}}

\DeclareMathOperator*{\argmax}{arg\,max}
\DeclareMathOperator*{\argmin}{arg\,min}

\usepackage{varwidth}
\newsavebox\tmpbox

\usepackage{spacingtricks}

\usepackage{iclr2022_conference}
\iclrfinalcopy

\usepackage[compact]{titlesec}
\def\setstretch#1{\renewcommand{\baselinestretch}{#1}}
\title{Learning Representations that Support\\ Robust Transfer of Predictors}

\author{Yilun Xu \\
Massachusetts Institute of Technology\\
\texttt{ylxu@mit.edu} \\
\And
Tommi Jaakkola \\
Massachusetts Institute of Technology\\
\texttt{tommi@csail.mit.edu}\\
}
\begin{document}

\maketitle

\begin{abstract}
Ensuring generalization to unseen environments remains a challenge. Domain shift can lead to substantially degraded performance unless shifts are well-exercised within the available training environments. We introduce a simple robust estimation criterion -- transfer risk -- that is specifically geared towards optimizing transfer to new environments. Effectively, the criterion amounts to finding a representation that minimizes the risk of applying any optimal predictor trained on one environment to another. The transfer risk essentially decomposes into two terms, a direct transfer term and a weighted gradient-matching term arising from the optimality of per-environment predictors. Although inspired by IRM, we show that transfer risk serves as a better out-of-distribution generalization criterion, both theoretically and empirically. We further demonstrate the impact of optimizing such transfer risk on two controlled settings, each representing a different pattern of environment shift, as well as on two real-world datasets. Experimentally, the approach outperforms baselines across various out-of-distribution generalization tasks. Code is available at \url{https://github.com/Newbeeer/TRM}.
\end{abstract}
\section{Introduction}
Training and test examples are rarely sampled from the same distribution in real applications. Indeed, training and test scenarios often represent somewhat different domains. Such discrepancies can degrade generalization performance or even cause serious failures, unless specifically mitigated. For example, standard empirical risk minimization approach (ERM) that builds on the notion of matching training and test distributions rely on statistically informative but non-causal features such as textures~\citep{Geirhos2019ImageNettrainedCA}, background scenes~\citep{Beery2018RecognitionIT}, or word co-occurrences in sentences~\citep{Chang2020InvariantR}. 

Learning to generalize to domains that are unseen during training is a challenging problem. One approach to domain generalization or out-of-distribution generalization is based on reducing variation due to sets or environments one has access to during training. For example, one can align features of different environments~\citep{Muandet2013DomainGV,Sun2016DeepCC} or use data-augmentation to help prevent overfitting to environment-specific features~\citep{Carlucci2019DomainGB,Zhou2020DeepDI}. At one extreme, domain adaptation assumes access to unlabeled test examples whose distribution can be then matched in the feature space (\eg \citep{Ganin2016DomainAdversarialTO}). 

More recent approaches build on causal invariance as the foundation for out-of-distribution generalization. The key assumption is that the available training environments represent nuisance variation, realized by intervening on non-causal variables in the underlying Structural Causal Model~\citep{Pearl2000CausalityMR}. Since causal relationships can be assumed to remain invariant across the training environments as well as any unseen environments, a number of recent approaches~\citep{Peters2015CausalIU,Arjovsky2019InvariantRM,Krueger2020OutofDistributionGV} tailor their objectives to remove spurious~(non-causal) features specific to training environments. 

In this paper, we propose a simple robust criterion termed Transfer Risk Minimization~(TRM). The goal of TRM is to directly translate model's ability to generalize across environments into a learning objective. As in prior work, we decompose the model into a feature mapping and a predictor operating on the features. Our transfer risk in this setting measures the average risk of applying the optimal predictor learned in one environment to examples from another adversarially chosen environment. The feature representation is then tailored to support such robust transfer. Although our work is greatly inspired by IRM~\citep{Arjovsky2019InvariantRM}, we show that TRM serves as a better out-of-distribution criterion with both empirical and theoretical analysis in non-linear case. We further show that the TRM objective decomposes into two terms, direct transfer term and a weighted gradient-matching term with connections to meta-learning. We then propose an alternating updating algorithm for optimizing TRM. 

To evaluate robustness we introduce two patterns of environment shifts based on 10C-CMNIST and SceneCOCO datasets. We construct these controlled settings so as to exercise different combinations of invariant and non-causal features, highlighting the impact of non-causal features in the training environments. In the absence of non-causal confounders, we show that all the methods achieve decent out-of-distribution generalization. When non-causal features are present, however, TRM offers greater robustness against biased training environments. We further demonstrate that our approach leads to good performance on the two real-world datasets, PACS and Office-Home. 

\section{Background and related works}
\paragraph{Domain generalization}
Machine learning models trained with Empirical Risk Minimization may not 
perform well in unseen environments where examples are sampled from a distribution different from training. The problem is known as out-of-distribution generalization or domain generalization~\citep{Blanchard2011GeneralizingFS,Muandet2013DomainGV}. A number of recent approaches have been proposed in this context. We only touch some of them for brevity.  
A typical approach to out-of-distribution generalization involves (distributionally) aligning training environments~\citep{Muandet2013DomainGV,Ganin2016DomainAdversarialTO,Sun2016DeepCC,Li2018DomainGV,Shi2021GradientMF}. Related approaches such as \cite{Nam2019ReducingDG} encourage the model to focus more on shapes via style adversarial learning, adopt data augmentations~\citep{Carlucci2019DomainGB,Zhou2020DeepDI} or meta-learning~\citep{Li2018LearningTG}. 
\vspace{-5pt}
\paragraph{Causal invariance} A recent line of work focuses on promoting invariance as a way to isolate causally meaningful features. Ideally, one would specify a structural equation model~\citep{Pearl2000CausalityMR}, expressing direct and indirect causes, distinguishing them from spurious, environment specific influences that are unlikely to generalize~\citep{Peters2015CausalIU,RojasCarulla2018InvariantMF,Mller2020LearningRM}. Invariance serves as a statistically more amenable proxy criterion towards identifying causally relevant features for predictors. 
\citet{Arjovsky2019InvariantRM} proposed \emph{invariant risk minimization} over feature-predictor decompositions. The main idea is that the predictor operating on causal features can be assumed to be simultaneously optimal across training environments. A number of related approaches have been proposed. For example, \cite{Krueger2020OutofDistributionGV} uses variance of losses as regularization, \cite{Jin2020DomainEV} minimizes the regret loss induced by held-out environments and \cite{Parascandolo2020LearningET} aligns gradient signs across environments by and-mask.\vspace{-5pt}
\paragraph{Distributionally robust optimization (DRO)} DRO specifies a minimax criterion for estimating predictors where an adversary gets to modify the training distribution. 
The allowed modifications are typically expressed in terms of divergence balls around the training distribution~\citep{BenTal2013RobustSO,Duchi2016StatisticsOR,Esfahani2018DatadrivenDR}. Closer to our work, Group DRO~\citep{Hu2018DoesDR,Sagawa2019DistributionallyRN} defines uncertainty regions in terms of a simplex over (fixed) training groups. Both DRO and Group DRO minimize the worst-case loss of the predictor within the uncertain regions. Unlike these methods, we use a predictor-representation decomposition, and define a regularizer over the representation using a minimax criterion. Moreover, we explicitly measure the risk of a predictor trained in one environment but applied to another.  


\section{Transfer Risk Minimization}
Consider a classification problem from input space $\gX$ (e.g. images) to output space $\gY$ (e.g. labels). We are given $E$ training environments  $\Omega=\{{P}_1(\gX \times \gY),\dots,{P}_E(\gX \times \gY)\}$, where $P_i$ is the empirical distribution for environment $i$. We decompose our model into two parts: feature extractor $\Phi:\gX \to \gZ$, which maps the input to a feature representation, and predictor $w:\gZ\to \gY$ that operates on the features to realize the final output. We call their concatenation $w\circ \Phi$ as a classifier. We use  $\ell(w \circ \Phi(x);y)$ to denote the cross-entropy loss on a training point $(x,y)\in \gX \times \gY$. As a shorthand, the expected loss with respect to a distribution $P$ is given as $\E_P[\ell(w \circ \Phi)]$. The broader goal is to learn a pair of feature extractor and predictor that minimize the risk on some unseen environment $\hat{P}$:
\begin{align*}
    \gR(\Phi,w) = \E_{\hat{P}}[\ell (w\circ\Phi)]
\end{align*}
As a step towards this goal, we learn a predictively robust model $(\Phi,w)$ across the available training environments (defined later). While the high level aim here resembles invariant risk minimization~\citep{Arjovsky2019InvariantRM}, our proposed estimation criterion is based on robustness rather than invariance. 
\subsection{Estimation criterion} 
We define group (environment) robustness based on exchangeability of predictors. Specifically, we require that environment-specific predictors $w$ generalize also to other training environments. Note that this doesn't imply that a single predictor is per-environment optimal or invariant as in IRM. Instead, our representation $\Phi$ aims to minimize \emph{transfer risk} across a set of training environments $\Omega$ 
\begin{align*}
    \gR(\Phi;\Omega)=\sum_{Q\in \Omega}\left(\sup\limits_{P \in \textrm{Conv}(\Omega \backslash Q)}\E_{P}[\ell(w(Q;\Phi) \circ \Phi)]\right)\numberthis\label{eq:obj-robust}
\end{align*}
where $w(Q;\Phi) = \argmin\limits_{w}\E_{Q}[\ell(w \circ \Phi)]$ refers to the optimal predictor $w$ with respect to distribution $Q$. $\textrm{Conv}(\Omega \backslash Q) = \left\{\sum_{P_i \in \Omega\backslash Q} \mathbf{\alpha}_i(Q)P_i\mid \alpha_i(Q) \ge 0, \parallel \mathbf{\alpha}(Q)\parallel_1=1\right\}$ is the convex hull of environment specific distributions, excluding $Q$. Unlike methods in the DRO family~\citep{BenTal2013RobustSO,Sagawa2019DistributionallyRN} that do not decompose the predictors, the robust estimation criterion here is specifically tailored to measure the goodness of features in terms of their ability to permit generalization across the environments. We will show in later sections that transfer risk (\Eqref{eq:obj-robust}) indeed ensures better out-of-distribution generalization.
\paragraph{Remark} We introduced transfer risk in \Eqref{eq:obj-robust} as a ``sum-sup'' criterion with respect to outer and inner terms. Other possible versions with similar estimation consequences include sum-sum, \ie $\gR(\Phi;\Omega)=\sum\limits_{Q\in \Omega}\left(\sum\limits_{P \in \textrm{Conv}(\Omega \backslash Q)}\E_{P}[\ell(w(Q;\Phi) \circ \Phi)]\right)$. Note that the criterion still measures whether the feature representation allows a predictor trained in one environment to generalize to another. We expect this version to behave similarly when training environments have comparable noise levels, complexities. However, sum-sum version can be more resistant to environmental outliers. 

\subsection{Comparison with IRM}

\label{sec:irm}
IRM~\citep{Arjovsky2019InvariantRM} is a popular objective for learning features that are invariant across training environments. Specifically, IRM finds a feature extractor such that the associated predictor is simultaneously optimal for every training environment. In our notation 
\begin{align*}
    \textrm{(IRM)} \quad \min\limits_{\Phi,w} \sum_{P\in \Omega}\E_{P}[\ell(w \circ \Phi)] \qquad \textrm{subject to  } w \in \argmin_w \E_{P}[\ell(w \circ \Phi)], \forall P \in \Omega
\end{align*}
IRM specifies a more restrictive set of admissible feature extractors $\Phi$ than transfer risk. Specifically, per-environment optimal predictors in IRM must agree (contain a common predictor) whereas transfer risk uses the per-environment optimal predictor to guide the representation learning. 
Due to the difficulty of solving the IRM bi-leveled optimization problem, \citet{Arjovsky2019InvariantRM} introduced a relaxed objective called IRMv1 where the constraints are replaced by gradient penalties: 
\begin{align*}
   \textrm{(IRMv1)} \quad \min\limits_{\Phi,w} \quad \sum_{P\in \Omega}\E_{P}[\ell(w \circ \Phi)] + \lambda \parallel \nabla_w \E_{P}[\ell(w \circ \Phi)]\parallel^2\numberthis\label{eq:irmv1}
\end{align*}
To compare with IRM, we use the theoretical framework in \citet{Rosenfeld2020TheRO}. For each environment, the data are defined by the following process: the binary label $y$ is sampled uniformly from $\{\pm 1\}$ and the environmental features $[z_c,z_e]$ are sampled subsequently from label-conditioned Gaussians:
\[z_c \sim \gN(y*\mu_c,\sigma_c^2 I_{d_c}) ; z_e \sim \gN(y*\mu_i,\sigma_e^2 I_{d_e}),i \in \{1,\dots,E\}\]
with $\mu_c \in \mathbb{R}^{d_c}, \mu_e \in \mathbb{R}^{d_e}$. The invariant feature mean $\mu_c$ remains the same for all environments while non-causal means $\mu_i$s vary across environments. The observation $x$ is generated as a function of the latent features: $x=f(z_c,z_e)$, where $f$ is a injective function that maps low dimensional features to high dimensional observations $x$. 

Theorem 3.3 in \citet{Rosenfeld2020TheRO} shows that for non-linear $f$, there exists a non-linear classifier $(\Phi,w)$ that has nearly optimal IRMv1 loss. In addition, it is equivalent to ERM solution on nearly all test points when the non-causal mean in the test environment is sufficiently different from those in training. Below we show that TRM can avoid the failure mode of IRM.


\begin{theorem}[Informal]
Under some mild assumptions, there exists a classifier that achieves near-optimal IRMv1 loss~(\Eqref{eq:irmv1}) and has high transfer risk~(\Eqref{eq:obj-robust}). In addition, for any test environment with a non-causal mean far from those in training, this classifier behaves like an ERM-trained classifier on most fractions of the test distribution.

\label{thm:irm2trm}
\end{theorem}
We defer the formal statement and proof to Appendix~\ref{proof:irm2trm}. We prove the above theorem by constructing a classifier \textit{only} using invariant features for prediction on the high-density region but behaving like ERM solution on the tails, which can still have near-optimal IRMv1 loss. However, the per-environment optimal predictors are distinct when using ERM-solution on the tails. The discrepancy in the per-environment optimal predictors leads to large transfer risk.

\begin{table}[t]

	\begin{minipage}{0.57\linewidth}
		\centering
		\subfigure{\includegraphics[width=0.5\linewidth]{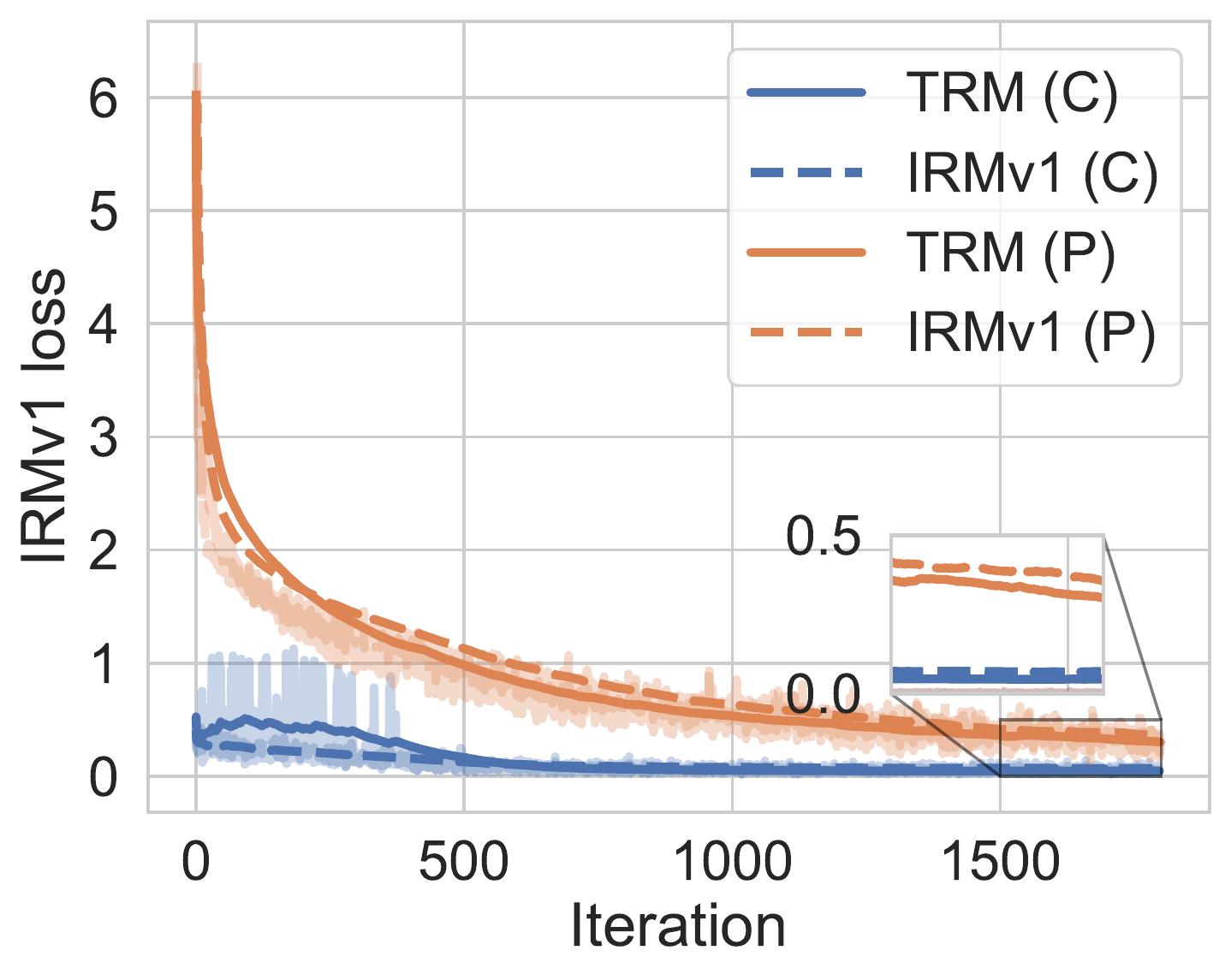}\label{img:irmv1}}%
   \subfigure{\includegraphics[width=0.5\linewidth]{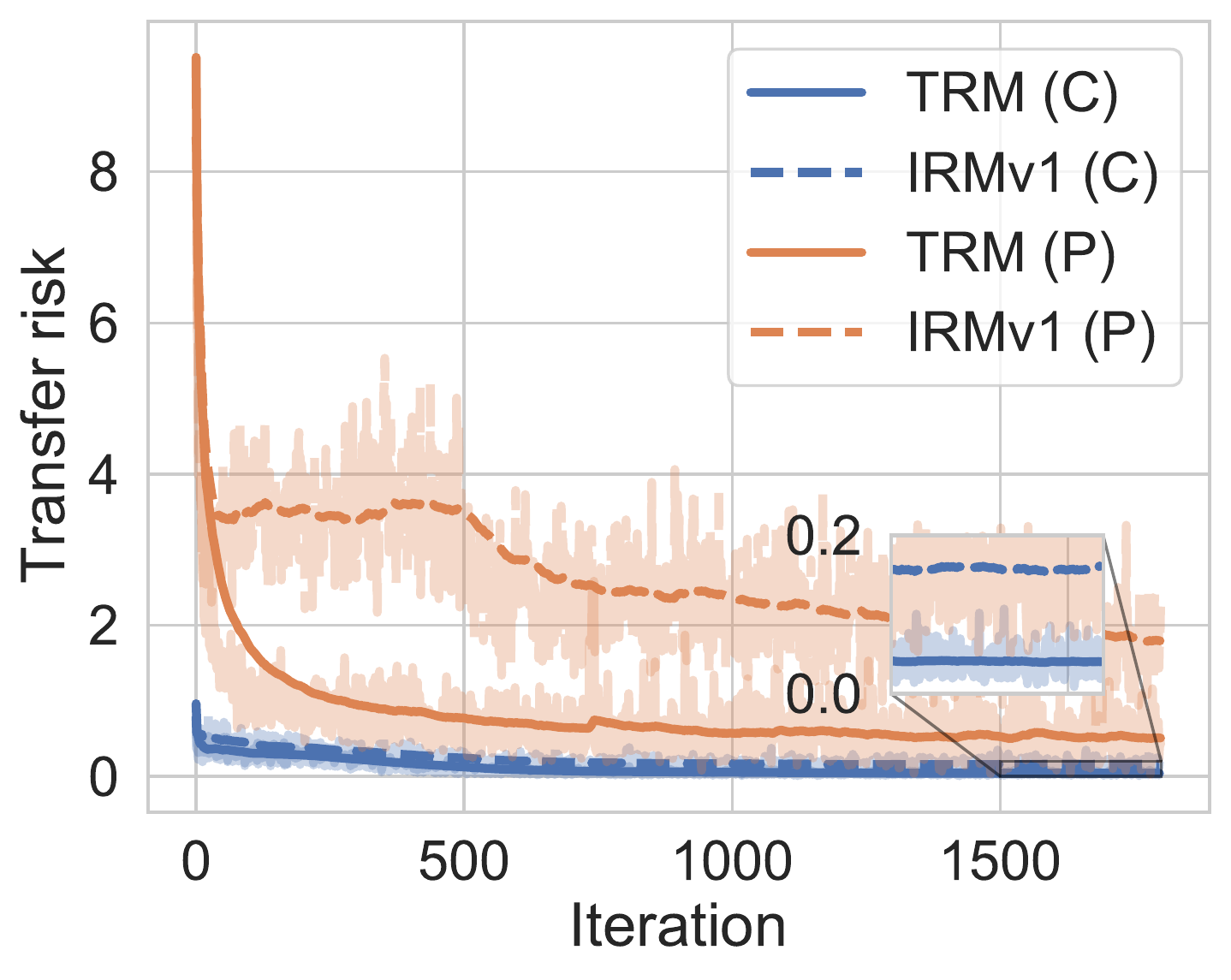}\label{img:trm}}
   \captionof{figure}{IRMv1 loss~(Left) and transfer risk~(Right) versus training iterations for models trained by IRMv1 and TRM on 10C-CMNIST~(C) and PACS~(P).}\label{img:analysis}
	\end{minipage}\hfill\hspace{5pt}
\begin{minipage}{0.42\linewidth}
		\centering
		\subfigure{\includegraphics[width=1\linewidth]{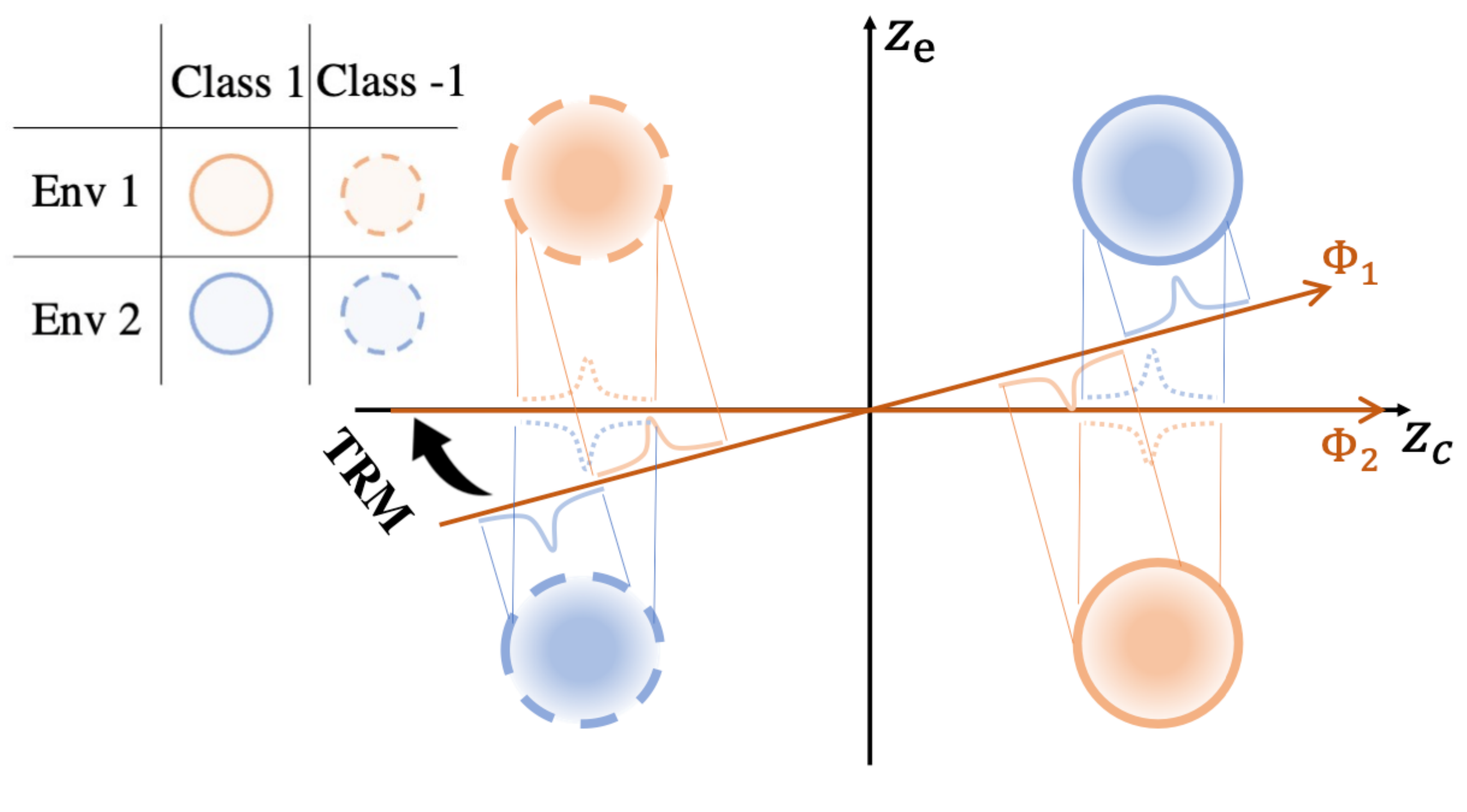}}%
   \captionof{figure}{2-d scenario of the linear case. TRM drives the non-invariant $\Phi_1$ toward the invariant $\Phi_2$.}\label{img:2d-toy}
	\end{minipage}
\end{table}

In addition to the theoretical analysis, we provide further empirical analysis to characterize the difference between TRM and IRM. \Figref{img:analysis} reports the IRMv1 loss and transfer risk on the 10C-CMNIST~(C) and PACS~(P) datasets discussed later~(\secref{sec:exp}). Although IRMv1 solutions achieve small IRMv1 losses, it has significantly higher transfer risks than TRM solutions. Conversely, TRM solutions have slightly lower IRMv1 loss than IRMv1 solutions. Besides, the out-of-distribution test accuracies on 10C-CMNIST / PACS are: $57\% / 73\%$ (IRMv1), $57\% / 74\%$~(ERM) and $78\% / 81\%$ (TRM). IRMv1 solutions have close performance to ERM solutions, while TRM outperforms others by a large margin. The empirical results support the statement in Theorem~\ref{thm:irm2trm} that models with near-optimal IRMv1 loss can have large transfer risks and behave like ERM solutions on test environments.

Together, these results suggest that transfer risk is a better criterion than IRMv1 for assessing model's out-of-distribution generalization. In \Figref{img:2d-toy}, we demonstrate the effect of TRM on a a toy 2-d example~($d_c=d_e=1$) with linear $f,\Phi$. TRM drives the non-invariant feature to the invariant one. We defer details of the analysis in the 2-d case to Appendix~\ref{app:toy}.


\section{Method}
In this section, we discuss how to optimize the TRM objective~(\Eqref{eq:obj-robust}). We first introduce an exponential gradient ascent algorithm for optimizing the inner supremum, and then discuss how to optimize the feature extractor with per-environment optimal predictor. An alternating updating algorithm incorporating these steps is summarized in \Algref{alg:trm}.
\subsection{Transfer Risk Optimization}
\paragraph{Solving the inner sup}
Assume $E$ different environments with associated densities $\Omega = \{P_1,P_2,\dots,P_E\}$. Given $Q$, we find the corresponding worst-case environment $P$ in the inner max of \Eqref{eq:obj-robust}. The search space for $P$ is the convex hull of all environment distributions with the exception of $Q$:
$P(Q)\in \textrm{Conv}(\Omega \backslash Q) = \left\{\sum_{P_i \in \Omega\backslash Q} \mathbf{\alpha}_i(Q)P_i\mid \alpha_i(Q) \ge 0, \parallel \mathbf{\alpha}(Q)\parallel_1=1\right\}$. 
Since the optimization is over a simplex, the solution can be found exactly by just selecting the worst environment in $\Omega$: $P=\argmax_{P\in \Omega}\E_{P}[\ell(w(Q) \circ \Phi)]$. Empirically, we find that updating $\alpha$ by gradient ascent instead of selecting the worst environment leads to a more stable training process. This has been observed in related contexts \citep{Sagawa2019DistributionallyRN}.

The gradient for $\mathbf{\alpha}_i$ is $\E_{P_i}[\ell(w(Q) \circ \Phi)]$, indicating that the inner supremum simply up-weights the environments with larger losses relative to the predictor $w(Q)$. We adopt an exponential gradient ascent algorithm (EG) for the updates: 
\begin{align*}
    &\mathbf{\alpha}_i(Q) = \textrm{EG}(\mathbf{\alpha}(Q),\eta_{\mathbf{\alpha}}, \Ls(Q))_i\numberthis \label{eq:eg}
    = \mathbf{\alpha}_i(Q)\exp(\eta_{\mathbf{\alpha}}\frac{\partial \Ls(Q)}{\partial \mathbf{\alpha}_i(Q)}) / \sum_{P_i \in  \Omega\backslash Q} \mathbf{\alpha}_i(Q)\exp(\eta_{\mathbf{\alpha}}\frac{\partial \Ls(Q)}{\partial \mathbf{\alpha}_i(Q)})
\end{align*}
where $\eta_{\alpha}$ is the learning rate, and the subscript $i$ denotes the $i$th component of the vector.

\paragraph{Updating the feature extractor $\Phi$} Given $Q$ and the corresponding worst-case environment $P(Q)$, we consider here how to update the feature extractor $\Phi$ in \Eqref{eq:obj-robust}. Denote the risk of using predictor $w(Q;\Phi)$ with data distribution $P(Q)$ as $\Ls_P(Q)=\E_{P(Q)}[\ell(w(Q;\Phi) \circ \Phi)]$. Recall that the optimal predictor $w(Q;\Phi)$ is an implicit function of the feature extractor $\Phi$, \ie $
    w(Q;\Phi) = \argmin_w   \E_{Q}[\ell(w \circ \Phi)]
$. In the remainder, we use a shorthand $w(Q)$ to refer to the predictor $\mathtt{sg}(w(Q;\Phi))$, where $\mathtt{sg}$ stands for the \texttt{stop\_gradient} operator. In other words, the value of $w(Q)$ follows $\Phi$ but it's partial derivatives w.r.t $\Phi$ are set to zero. It is helpful to distinguish $w(Q)$ from $w(Q;\Phi)$ to clarify what is meant by the different expressions below. Now, 
the full gradient of the transfer risk w.r.t $\Phi$ comprises two terms since $w(Q;\Phi)$ also depends on $\Phi$ 
\begin{align*}
    &\frac{\diff \Ls_P(Q)}{\diff \Phi}  = \underbrace{\frac{\partial\Ls_P(Q)}{\partial \Phi}}_\textrm{{direct gradient}}+ \underbrace{(\frac{\partial \Ls_P(Q)}{\partial w(Q;\Phi)})^T \frac{\diff w(Q;\Phi)}{\diff \Phi}}_\textrm{{implicit gradient}}\label{eq:phi-decom}\numberthis
\end{align*}
We show in Proposition~\ref{prop:implicit} that the implicit gradient can be further simplified as a weighted gradient-matching term.
\begin{restatable}{proposition}{implicit}
\label{prop:implicit}
Denote the Hessian as $H_{w(Q)}=\frac{\partial^2 \E_{Q}[\ell(w(Q) \circ \Phi)]}{\partial w(Q)^2}$. Suppose the loss $\ell(w \circ \Phi)$ is continuously differentiable and $H_{w(Q)}$ is non-singular, then we have:
\begin{align*}
    &(\frac{\partial \Ls_P(Q)}{\partial w(Q;\Phi)})^T \frac{\diff w(Q;\Phi)}{\diff \Phi} = -\frac{\partial (\mathtt{sg}(v_Q)^T\frac{\partial \E_{Q}[\ell(w(Q) \circ \Phi )]}{\partial w(Q)})}{\partial \Phi}
\end{align*}
where $v_Q=(\frac{\partial \Ls_P(Q)}{\partial w(Q)})^T H_{w(Q)}^{-1}$, and is treated as a constant vector in the above equation (note the use of stop-gradient version $w(Q)$).
\end{restatable}
We can interpret the numerator of RHS as a gradient-matching objective: it measures the similarity of gradients depending on whether the distribution over which the loss is measured is $Q$ or $P(Q)$, weighted by the Hessian inverse $H_{w(Q)}^{-1}$. It shows that TRM naturally aims to find a representation $\Phi$ where the gradients are matched when moving from $Q$ to $P$ (cf. ~\citep{Shi2021GradientMF}). By integrating, we can write down an objective function whose gradient with respect to $\Phi$ matches Eq.~\ref{eq:phi-decom}: 
\begin{align*}
  \int \frac{\diff \Ls_P(Q)}{\diff \Phi} \diff\Phi =\underbrace{\E_{P}[\ell(w(Q) \circ \Phi)]}_\textrm{{direct transfer term}}\hspace{5pt}-\hspace{5pt} \underbrace{\mathtt{sg}(v_Q)^T\frac{\partial \E_{Q}[\ell(w(Q) \circ \Phi )]}{\partial w(Q)}}_\textrm{\vspace{-100pt}{weighted gradient-matching term}}\hspace{5pt} + \underbrace{C}_\textrm{\vspace{-5pt}{constant term}} \label{eq:int}\numberthis
\end{align*}
Note the use of stop-gradient versions $w(Q)$ in these expressions. Effectively, the TRM objective for $\Phi$ decomposes into two terms: (i) the direct transfer term, which encourages predictor $w(Q)$ to do well even if the distribution were $P(Q)$, and (ii) the weighted gradient-matching term. The second term attempts to match the gradient of $w$ in the original environment $Q$ and worst-case environment $P(Q)$ by updating the features. The weighted gradient-matching term plays a role analogous to meta-learning~\citep{Li2018LearningTG,Shi2021GradientMF} encouraging simultaneous loss descent. Note that the weighted gradient-matching term actually evaluates to zero at the current value of $\Phi$ since $w(Q)$ is set to the per-environment optimal value, but the gradient of this term with respect to $\Phi$ is not zero.


\subsection{Approximation of Inverse Hessian Vector Product}

We denote the number of parameters in $w$ as $P$ and the gradient $\frac{\partial \Ls_P(Q)}{\partial w(Q)}$ as $v$. Dropping the $w({Q})$ subscript for clarity, the weight gradient matching term in \Eqref{eq:int} involves the computation of inverse Hessian vector product $v_Q=H^{-1}v$. For minibatch data $\mathcal{B}(x)$, computing the inverse Hessian $H^{-1}$ requires $O(|\mathcal{B}(x)|P^2 + P^3)$ operations. To avoid heavy computation, we use the similar approach in \citet{Agarwal2017SecondOrderSO} to get good approximations by Taylor expansion and efficient Hessian-vector product~(HVP)~\citep{Pearlmutter1994FastEM}. Let $H_{j}^{-1} = \sum_{i=0}^j(I-H)^i$ be the first $j$ terms in the Taylor expansion of $H^{-1}$. Note that $\lim_{j\to \infty}H_{j}^{-1}=H^{-1}$. We can solve the corresponding matrix vector product $H_{j}^{-1}v = \sum_{i=0}^j(I-H)^iv$ in linear time by recursively computing $(I-H)^iv$ with fast HVP. These computation are easy to implement in auto-grad systems like PyTorch~\citep{Paszke2019PyTorchAI}. 
\subsection{Algorithm}
In addition to the TRM objective in \Eqref{eq:obj-robust}, we include standard ERM term $\E_Q[\ell (w_{all}\circ\Phi)]$ for updating the predictor $w_{all}$ on the top of features. Overall, given the distributions $Q,P(Q)$, the per-environment objective for updating the $(\Phi,w_{all})$ pair consists of three terms:
\begin{align*}
    \hspace{-5pt}\gR(\Phi,w_{all}; Q)=
    \E_Q[\ell (w_{all}\circ\Phi)]+ \E_{P}[\ell(w(Q) \circ \Phi)] -\lambda\mathtt{sg}(v_Q)^T\frac{\partial \E_{Q}[\ell(w(Q) \circ \Phi )]}{\partial w(Q)} \numberthis\label{risk:per-env}
\end{align*}
where $\lambda$ is a hyper-parameter for adjusting the gradient-matching term to have the same gradient magnitude as the other terms. We interleave gradient updates on the model parameters ($\Phi,w_{all}$) and the environmental weights $\{\mathbf{\alpha}(Q)\mid Q\in \Omega\}$, as shown in \Algref{alg:trm}. 

\begin{algorithm}[htb]
   \caption{TRM algorithm}
   \label{alg:trm}
\begin{algorithmic}
   \STATE {\bfseries Input:} Inital model parameters $\Phi,w_{all}$, learning rates $\eta_{\phi},\eta_{w},\eta_{\mathbf{\alpha}}$. and environment set $\Omega$
   \FOR{$t=1$ {\bfseries to} $T$}
    \STATE Randomly pick a environment $Q\in \Omega$
    \STATE Get the optimal $w(Q)$ on $Q$
    \STATE Update the model parameters:
    
    \quad $\Phi^t \gets\Phi^{t-1} -\eta_{\phi}\nabla \gR(\Phi,w_{all};Q)$, $w_{all}^t \gets w_{all}^{t-1} - \eta_{w}\nabla \gR(\Phi,w_{all};Q) $ 
    \vspace{2pt}
    \STATE Update the environmental weights by \Eqref{eq:eg}:
    
    \quad $\mathbf{\alpha}(Q) \gets \textrm{EG}(\mathbf{\alpha}(Q),\eta_{\mathbf{\alpha}}, \Ls(Q))$
   \ENDFOR
\end{algorithmic}
\end{algorithm}

\Algref{alg:trm} updates $\alpha(Q)$ in an online manner. With some convexity, boundness, and smoothness assumptions, we can prove that the on-line updating has a convergence rate of $\mathcal{O}({1}/{\sqrt{T}})$ by using the techniques in \citet{Nemirovski2009RobustSA}. We defer more discussions to Appendix~\ref{app:converge}.

\section{Experiments}
\label{sec:exp}
In our experiments, we focus on the out-of-distribution generalization tasks. We first evaluate our method on two synthesized datasets (10C-CMNIST, SceneCOCO). We simulate three kinds of domain shifts by controlled experiments. Next, we evaluate all the methods on two real-world datasets (PACS, Office-Home). We compare \textbf{TRM} with standard empirical risk minimization (\textbf{ERM}), and recent methods developed for out-of-distribution generalization: \textbf{IRM}~\citep{Arjovsky2019InvariantRM}, \textbf{REx}~\citep{Krueger2020OutofDistributionGV},  \textbf{GroupDRO}~\citep{Sagawa2019DistributionallyRN}, \textbf{MLDG}~\citep{Li2018LearningTG} and \textbf{Fish}~\citep{Shi2021GradientMF}. We also use ERM trained with data sampled from the test domain to serve as an upper bound (\textbf{Oracle}).

{Experiments in the main body use training-domain validation sets for hyper-parameter selection, which are arguably more practical for out-of-distribution generalization task~\citep{ahmed2021systematic,Gulrajani2020InSO,Krueger2020OutofDistributionGV}. We defer the results of the test-domain validation set to Appendix~\ref{app:extra-exp}. We also show the efficacy of TRM on group distributional robustness in Appendix~\ref{app:dro}.}
\subsection{Experiments on 10C-CMNIST and SceneCOCO}
\label{exp:synthetic}
\subsubsection{Datasets}
\label{exp:dataset}
Evaluating the out-of-distribution generalization performance in an unambiguous manner necessitates controlled experiments. We synthesize the data by three latent features: (i) invariant~(causal) feature, (ii) non-causal feature, which is spuriously correlated with labels and (iii) the dummy feature, which is not predictive of the labels. We conduct the controlled experiments on two synthetic datasets:
\vspace{-6pt}
\paragraph{10C(lasses)-C(olored)MNIST}\hspace{-5pt}is a more general 10 classes version of the 2-classes ColoredMNSIT~\citep{Arjovsky2019InvariantRM}. We add the digit colors and background colors to allow for the domain shifts. Specifically, we set the invariant/non-causal/dummy features to digit/digit color/background color respectively. We randomly select ten colors as the digit colors and five colors as the background colors. 10C-CMNIST contains 60000 datapoints of dimension (3,28,28) from 10 digit classes.
\vspace{-6pt}
\paragraph{SceneCOCO} \hspace{-5pt}superimposes the objects from the COCO datasets~\citep{Lin2014MicrosoftCC} on the background scenes from the Places datasets~\citep{Zhou2018PlacesA1}. Following~\citet{ahmed2021systematic}, we select 10 objects and 10 scenes from above two datasets. We set the invariant/non-causal/dummy features to object/background scene/object color. This dataset consists of 10000 datapoints of dimension $(3,64, 64)$ from 10 object classes.

In addition, we define a measurement of the correlations between the label and the non-causal features. Note that there is a one-to-one corresponding between the label and the non-causal features, \eg ``2" $\leftrightarrow$ ``blue digit color” in 10C-CMNIST and ``boat" $\leftrightarrow$ ``beach scene" in SceneCOCO. For each environment, we define the \textit{bias degree} to be the ratio of the data that obeys this relationship. Those data which don’t follow this relationship are then assigned with random non-causal features. In each training environment, the data is generated by environmental-specific combination of features and bias degree. This setting is commonly adopted in existing literature~\citep{Arjovsky2019InvariantRM, Krueger2020OutofDistributionGV,ahmed2021systematic}. The label $y$ is set as the class where the invariant feature lies. 
\subsubsection{Controlled Scenarios}
\label{sec:sce}
\begin{figure*}[tb]
\centering
  \subfigure[Label-correlated shift]{\includegraphics[width=0.48\linewidth]{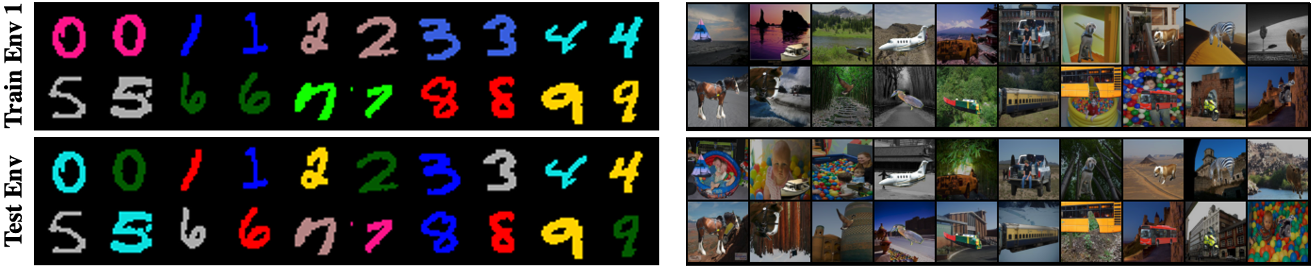}}\hfill%
    \subfigure[Combined shift]{\includegraphics[width=0.48\linewidth]{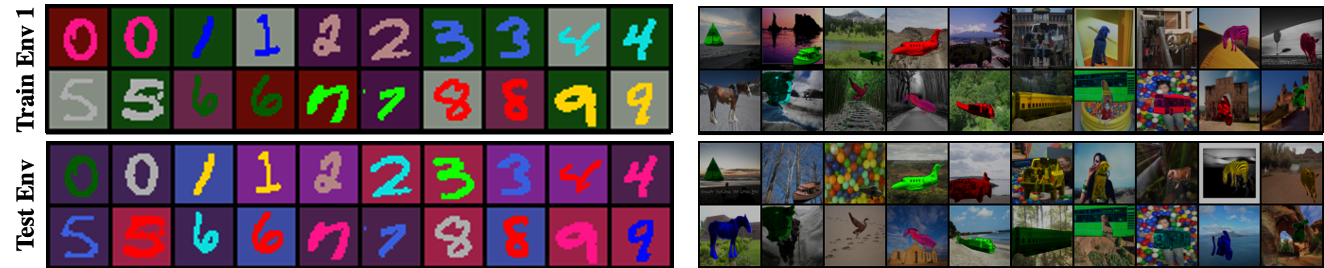}}\vspace{-5pt}
    \caption{Visualization of the training environment 1 (\textbf{top}) and the test environment (\textbf{bottom}) of 10C-CMNIST and SceneCOCO on (a) Label-correlated shift and (b) Combined shift.}
    \label{img:vis}
    \vspace{-10pt}
\end{figure*}

\begin{figure*}[t]
    \centering
    \subfigure[Label-correlated shift]{\includegraphics[width=0.33\linewidth]{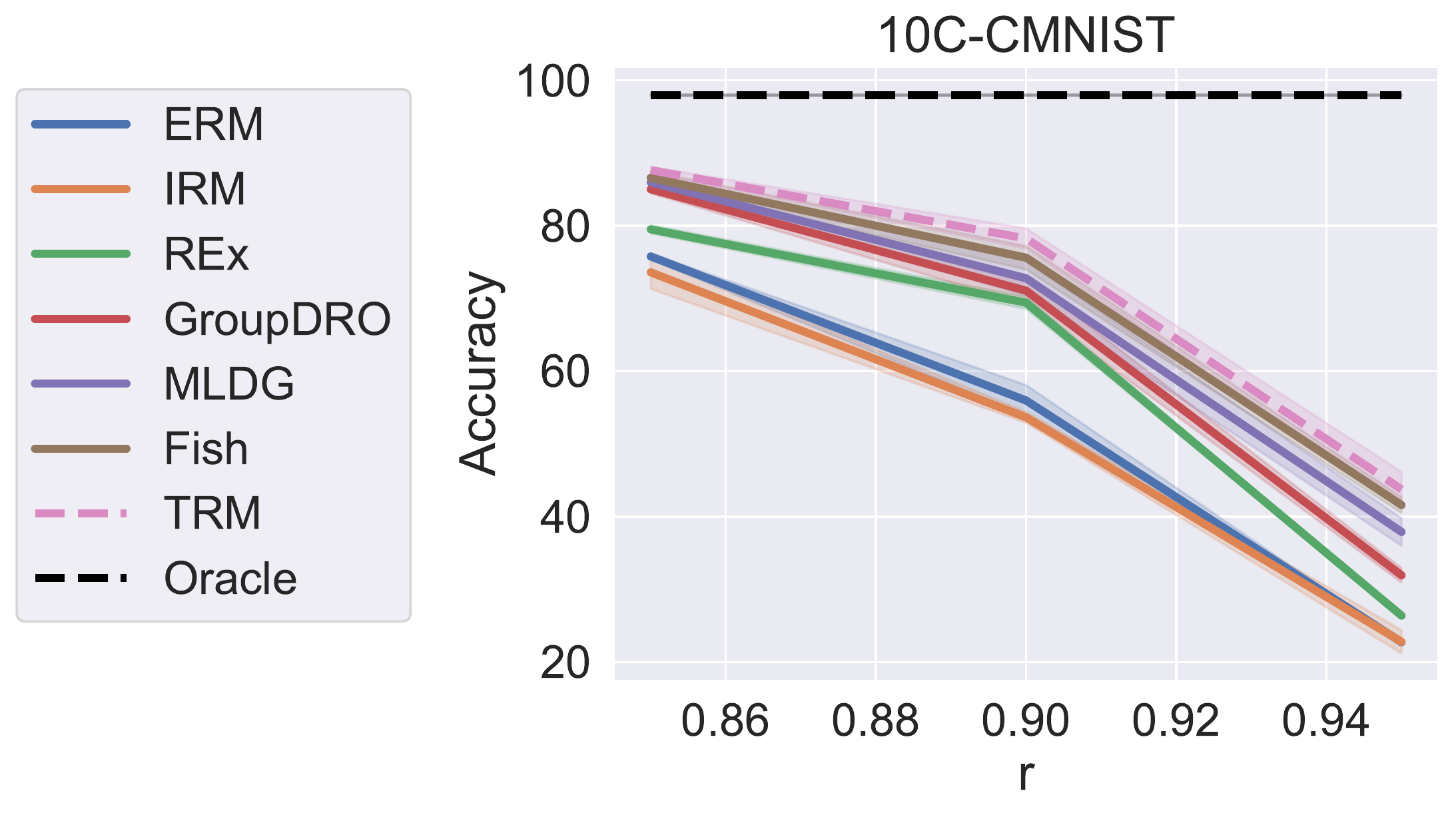}\includegraphics[width=0.225
    \linewidth]{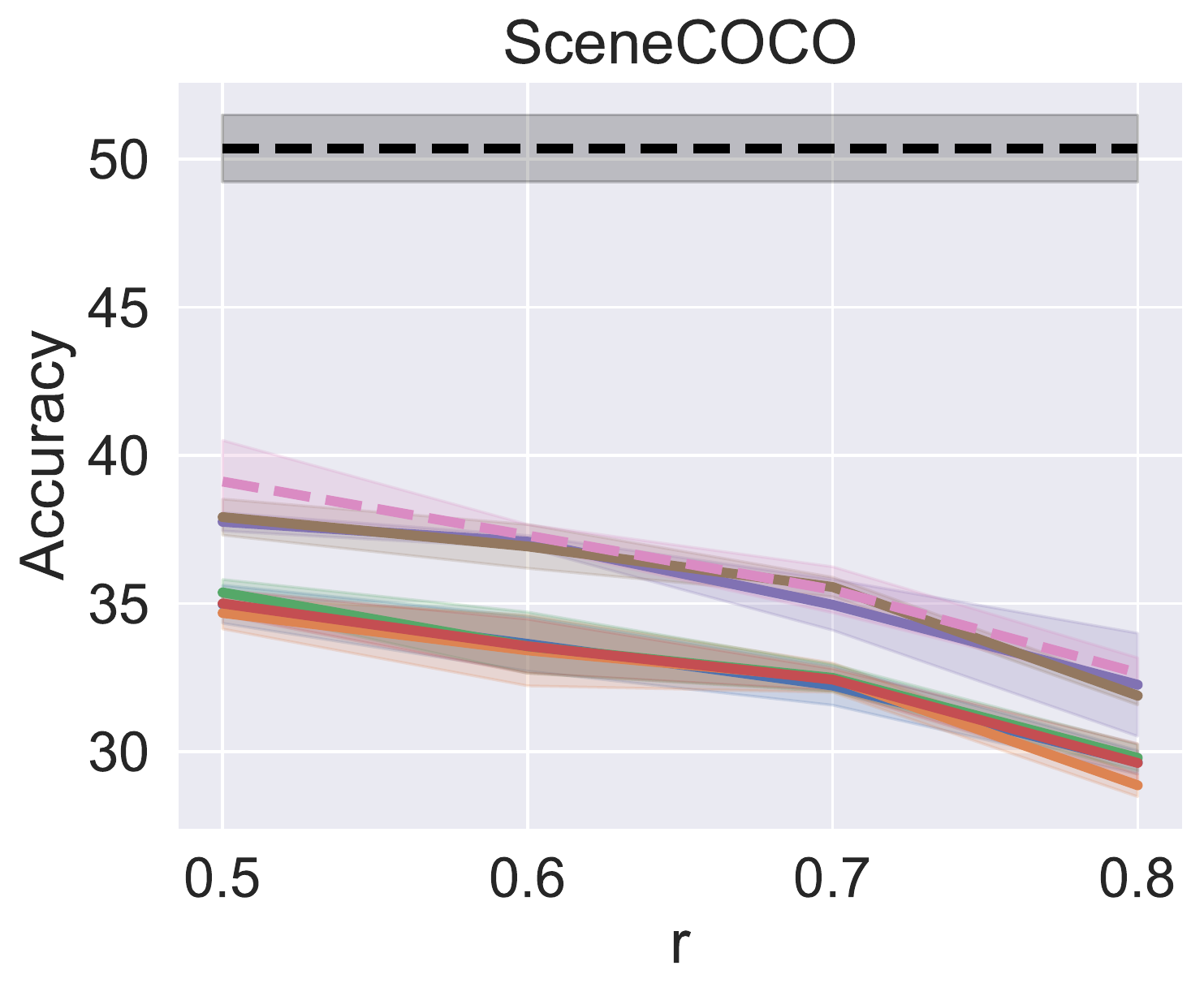}}%
   \subfigure[Combined shift]{\includegraphics[width=0.23\linewidth]{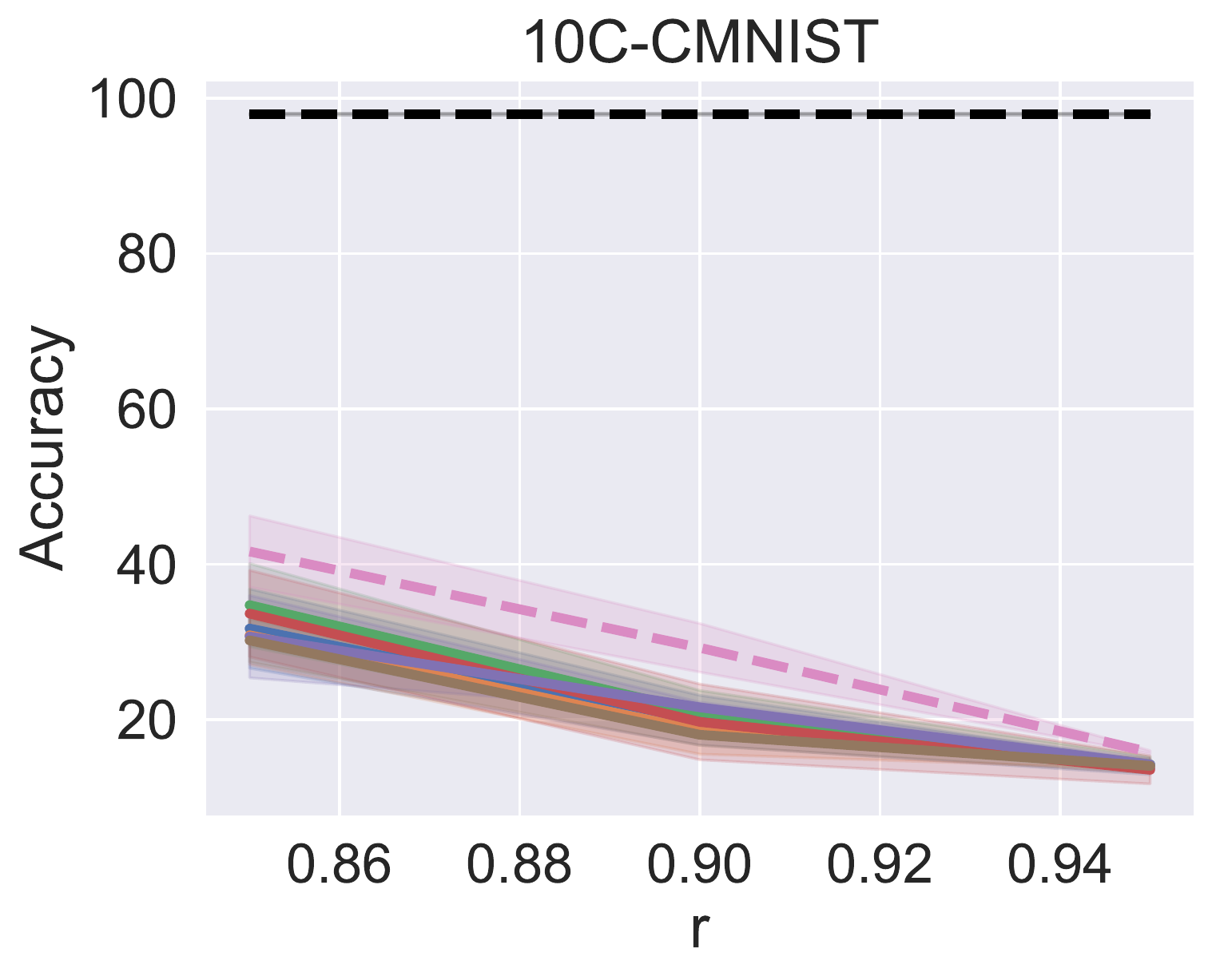}\includegraphics[width=0.225\linewidth]{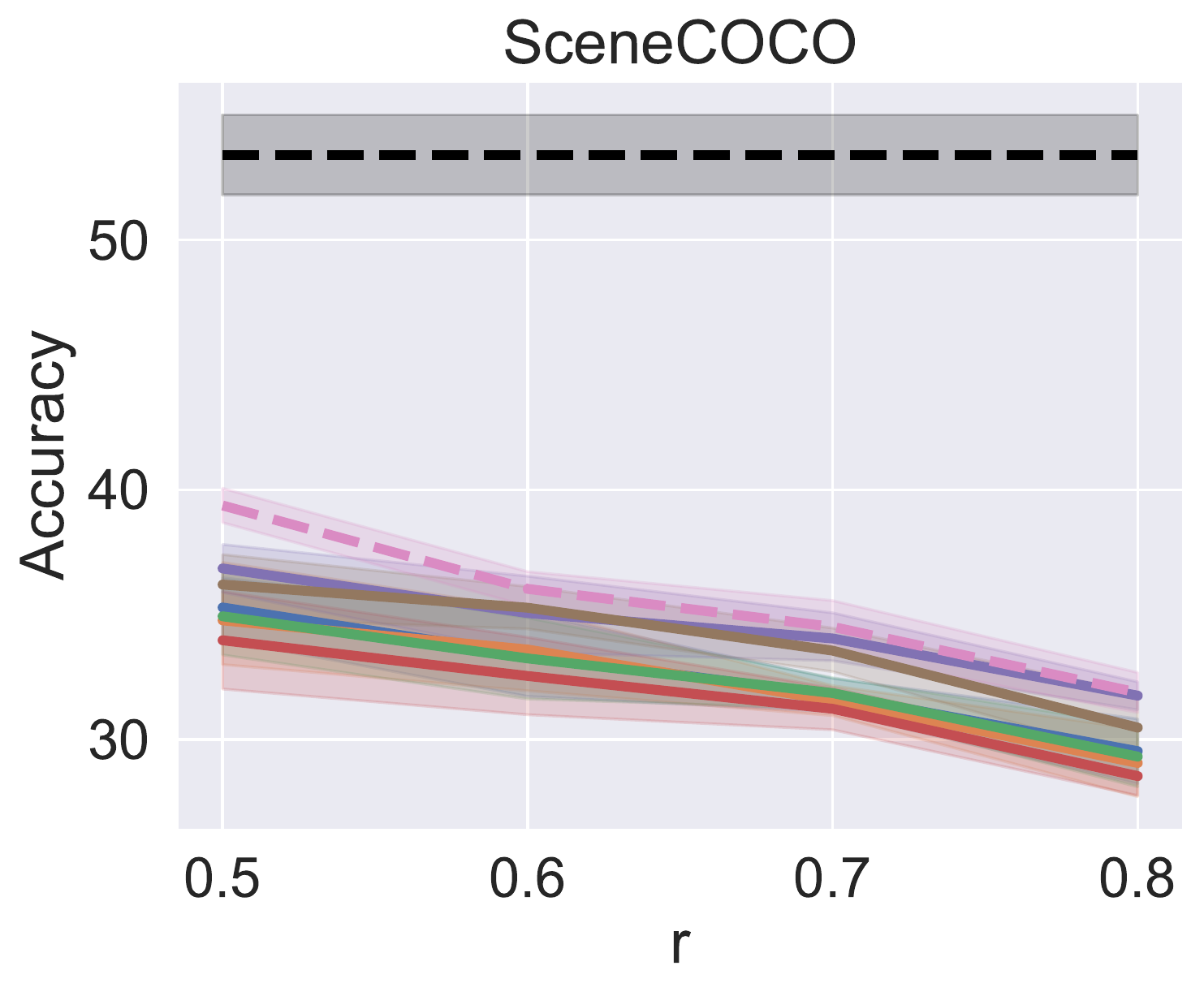}}%
    \caption{Test accuracy on 10C-CMNIST and SceneCOCO datasets in (a) Label-correlated shift and (b) Combined shift, using various bias degrees $r$. }
    \label{img:correlated}
    \vspace{-5pt}
\end{figure*}
Next, we consider two scenarios with distinct combinations of latent features.

\paragraph{Label-correlated shift}
In this scenario, each training environment is assigned with a different non-zero bias degree. The dummy feature is set to a constant, \emph{e.g.} black background color in 10C-CMNIST. The bias degree is set to zero in the test environment for evaluating how much extent the model has learn the invariant feature.

\vspace{-5pt}
\paragraph{Combined shift}
In this scenario, training environments have varying non-zero bias degrees and prior distributions of dummy feature. The test environment is unbiased. It simulates the joint effects of the shifts of non-causal and dummy features.

{\Figref{img:vis} visualizes the label-correlated and combined shifts. We use two training environments in the controlled experiments. To better evaluate the robustness of algorithms, we vary the bias degrees in the second training environment. We use $100\%$/$r\%$ bias degrees in 10C-CMNIST for the first/second training environment respectively, and $90\%$/$r\%$ in SceneCOCO. $r\%$ is ranging from $85\%$ to $95\%$ in 10C-CMNIST and $50\%$ to $80\%$ in SceneCOCO. We use different configurations for the two datasets because SceceCOCO is more complex. The biased degree is $0\%$ in the test environment.}

We adopt a 4-layer CNN/Wide ResNet~\citep{Zagoruyko2016WideRN} as the feature extractor for 10C-CMNIST/SceneCOCO, following prior work~\citep{ahmed2021systematic}. We train for 10/100 epochs on 10C-CMNIST/SceneCOCO, both using batch size 128 and SGD with $0.1$ initial learning rate and 0.9 momentum. 
For more details of datasets, hyper-parameter selection and training, please refer to Appendix~\ref{app:exp}. 

\subsubsection{Results}
In \Figref{img:correlated}, we report the accuracy on the test environment under label-correlated shift and combined shift. The x-axis in \Figref{img:correlated} stands for the varying biased degree $r\%$ of the second training environment. We observe a consistent performance drop of all the methods as the training environments become more biased. Our main finding is that the TRM algorithm achieves a better test accuracy at most bias degrees on both datasets and competitive performance with Fish and MLDG on SceneCOCO when the bias degrees are large. The results show that the model trained by the TRM algorithm depends more on the invariant features to make predictions. 

We also observe that non-causal features combined with the distribution shifts of dummy features degrade the performance of all the methods~(Combined shift). In Appendix~\ref{app:dummy}, we show that all the methods have similar performance to the Oracle when only changing the dummy feature distribution. The experiments suggest that when non-causal features exist, distribution shifts on dummy features can further hurt the out-of-distribution generalization. Besides, we show in Appendix~\ref{app:transfer} that TRM-trained models transfer faster to target environments with limited data for fine-tuning.

\subsection{Experiments on PACS and Office-Home}
\subsubsection{Setups}

We further evaluate our methods on two real-world datasets, PACS and Office-Home.

\textbf{PACS}~\citep{Li2017DeeperBA} comprises four environmental data, namely arts, cartoons, photos, and sketches. This dataset contains 9991 datapoints of dimension (3,224,224) from 7 classes. 

\textbf{Office-Home}~\citep{Venkateswara2017DeepHN} includes four environments, namely art, clipart, product and real. It contains 15588 datapoints of dimension (3,224,224) from 65 classes. 

These two datasets are widely used in domain generalization literature. We follow the standard valuation protocol~\citep{Li2017DeeperBA}, which reports the test accuracy on each hold-out environment when training on the other three environments. We use the ImageNet-pretrained ResNet18 as the backbone of feature extractors. We use the SGD optimizer with a momentum of $0.9$, a weight decay of 1e-4, and a fixed learning rate of 1e-4. The batch size is set to $32$.
\subsubsection{Results}

In Table~\ref{table:pacs-res18} and~\ref{table:office-res18}, we report the test accuracy on PACS and Office-Home dataset. The proposed TRM algorithm achieves superior average accuracy on these datasets. We observe that TRM has better generalization ability on most hold-out environments in the two datasets, except the Photo environment, where all the methods have comparable performance. Further, we test TRM without the weighted gradient-matching term~(TRM w/ GM) by setting $\lambda=0$. TRM w/ GM still outperforms other baselines with only the direct transfer term. We also show in Appendix~\ref{app:pacs-exp} that TRM consistently improves over other methods with different architectures and validation set configurations.

\begin{table*}[htbp]
\small
\begin{center}
\caption{Test accuracy on PACS dataset}
\label{table:pacs-res18}
\begin{tabular}{c c c c c c c c}
		\toprule
		\textbf{Algorithm} & \textbf{Art} & \textbf{Cartoon} & \textbf{Photo} & \textbf{Sketch}	& \textbf{Average}\\
		\midrule
        \multirow{1}{*}{ERM} & $73.7\pm 1.0$ & $65.7\pm 2.3$ & {$94.8\pm 0.7$} & $62.5\pm 2.4$ & $74.1$ \\
        \multirow{1}{*}{IRM}&$73.3\pm 1.0$ &$65.5\pm 1.9$&$94.7\pm 0.6$& $62.9\pm 1.5$ &$74.1$\\
        \multirow{1}{*}{REx}&$74.0\pm 1.0$&$66.8\pm 2.5$&$94.6\pm 0.8$&$63.7\pm 2.8$&$74.8$\\
        \multirow{1}{*}{GroupDRO}& $74.1\pm 0.8$& $67.3\pm 2.0$& {94.7 $\pm$ 0.7}& $63.5\pm 3.4$&$74.9$\\
        \multirow{1}{*}{MLDG}& $76.0\pm 1.8$& $69.2\pm 0.9$& \textbf{95.0 $\pm$ 0.4}& $64.3\pm 3.4$&$76.4$\\
        \multirow{1}{*}{Fish}& $75.0\pm 1.2$& $69.0\pm 1.7$& {94.7 $\pm$ 0.5}& $64.3\pm 1.6$&$76.0$\\
        \midrule
        TRM w/ GM & $77.5\pm 2.3$ & $\bf{70.3\pm 0.9}$& $94.4\pm0.7$ & $65.5 \pm 1.7$ & 76.9 \\
        TRM & {$\bf{80.6 \pm 2.1}$} & {$68.7 \pm 1.4$}& $93.7\pm 1.5$ & $\bf{67.1 \pm 2.5}$ & $\bf{77.5}$ \\
		\bottomrule
\end{tabular}
\end{center}
\vspace{-5pt}
\end{table*}

\begin{table*}[htbp]
\small
\begin{center}
\caption{Test accuracy on Office-Home dataset}
\label{table:office-res18}
\begin{tabular}{c c c c c c c c}
		\toprule
		\textbf{Algorithm} & \textbf{Art} & \textbf{Clipart} & \textbf{Product} & \textbf{Real}	& \textbf{Average}\\
		\midrule
        \multirow{1}{*}{ERM} & $51.1\pm 0.4$ & $42.5\pm 0.7$ & $65.5\pm 0.1$ & $68.4\pm 0.7$ & $56.8$ \\
        \multirow{1}{*}{IRM}&$51.6\pm 0.2$ &$42.5\pm 0.6$&$65.2\pm 0.1$& $68.1\pm 0.7$ &$56.8$\\
        \multirow{1}{*}{REx}&$50.9\pm 1.3$&$42.3\pm 0.6$&$65.1\pm 0.4$&$68.1\pm 0.7$&$56.6$\\
        \multirow{1}{*}{GroupDRO}& $51.2\pm 0.6$& $42.1\pm 0.9$& $64.9\pm 0.4$& $67.9\pm 0.6$&$56.5$\\
        \multirow{1}{*}{MLDG}& $52.3\pm 0.3$& $44.1\pm 0.7$& $66.9\pm 0.8$& $69.4\pm 0.5$&$58.2$\\
        \multirow{1}{*}{Fish}& $51.5\pm 0.5$& $43.2\pm 0.9$& $66.5\pm 0.3$& $69.1\pm 0.1$&$57.6$\\
        \midrule
        TRM w/ GM & $\bf{54.0\pm 0.5
        }$& $46.4\pm 0.9$ & $67.9\pm 1.0$& $69.4\pm 1.0$& $59.4$\\
        TRM & {${53.9 \pm 0.2}$} & {$\bf{46.4 \pm 0.9}$}& $\bf{68.9\pm 1.1}$ & $\bf{69.6 \pm 0.8}$ & $\bf{59.7}$ \\
		\bottomrule
\end{tabular}
\end{center}

\end{table*}

\section{Conclusion}
The discrepancy between the training and test domain can degrade the performance of algorithms developed for the i.i.d setting. We propose a robust criterion termed Transfer Risk Minimization~(TRM) to tackle the out-of-distribution problem. The transfer risk promotes the transferability of the per-environment predictors. The feature representation updates accordingly to support such transfer. We demonstrate that TRM better recovers the weights associated with the invariant features by an illustrative example. Due to the optimality of the per-environment predictor, TRM objective naturally decomposes into two terms, the direct transfer term and the weighted-gradient matching term. One limitation of TRM is that the inverse Hessian vector product can have a large variance with a small batch size. The better optimization of the weighted gradient-matching term is left for future work.

Experimentally, we test our approach on several controlled experiments. We show that TRM achieves better out-of-distribution performance under different combinations of features. We also demonstrate the effectiveness of TRM on the PACS and Office-Home datasets.

\section*{Acknowledgements}

This research was supported by the HDTV Grand Alliance Fellowship.
\clearpage
\bibliography{bib}
\bibliographystyle{iclr2022_conference}

\clearpage
\appendix
\section{Proofs}
\label{app:proofs}

\subsection{Proof of Theorem \ref{thm:irm2trm}}
\label{proof:irm2trm}
In the following theorem, we show a simplified version as in \citet{Rosenfeld2020TheRO}, where $\sigma_e = 1, \forall e \in \{1,2, \dots, E\}$. The full version can be similarly deduced.
\begin{theorem}[Formal statement]
Assume $\sqrt{d_e} \le \parallel \mu_c \parallel_2 \le 8\sqrt{d_e},\parallel \mu_i \parallel_2 \le 8\sqrt{d_e}, i \in \{1,\dots, E\} $. We further assume there exist two training environments $i,j\in \{1,2, \dots, E\}$ such that $\parallel \mu_i-\mu_k\parallel_2\ge 2\sqrt{2d_e}, \forall k \in \{1,2,\dots E\}-\{i\}$ and $ \mu_i^T\mu_j \le -\frac{1}{\sigma_c^2}\parallel \mu_c \parallel_2^2$. Then there exists a classifier which achieves near optimal IRMv1 loss~(\Eqref{eq:irmv1}) and has high transfer risk~(\Eqref{eq:obj-robust}) when $d_e>>1$. In addition, for any test environment $E+1$ with a non-causal mean far from the those in training:
\[\forall e \in \{1,\dots,E\}, \min_{y \in \{\pm 1\}}\parallel\mu_{E+1}-\mu_i\parallel\ge (\sqrt{2}+\delta)\sqrt{d_e}\]
for some $\delta>0$. Then the classifier behaves like the ERM-trained classifier on $1-\frac{2E}{\sqrt{\pi}\delta}\exp{(-\delta^2)}$ fraction of the test distribution.
\end{theorem}
\begin{proof}
We follow the proof idea of theorem 6.1 in \citet{Rosenfeld2020TheRO} and first define $r=\sqrt{2d_e}$. We construct $\gB_r$ as
\[\gB_r = \left[\bigcup_{e\in \{1,\dots,E\}-\{ i\}}B_r(\mu_e)\right] \cup \left[\bigcup_{e\in \{1,\dots,E\}-\{ i\}}B_r(-\mu_e)\right]\]
where $B_r(\mu_e)$is the $\ell_2$ ball centered at $\mu_e$. We construct the classifier as follows:
\[\Phi = \left\{
\begin{aligned}
\begin{bmatrix}
z_c \\ 0
\end{bmatrix}, z_e \in \gB_r \\
\begin{bmatrix}
z_c \\ z_e
\end{bmatrix}, z_e \in \gB_r^c
\end{aligned}
\right. \quad \text{and} \quad w = \begin{pmatrix}
\frac{2\mu_c}{\sigma_c^2} \\ {2\mu_i}
\end{pmatrix}\]
$\Phi$ outputs the invariant feature $\begin{bmatrix}
z_c \\ 0
\end{bmatrix}$ in the $\ell_2$ balls centered at non-causal means except $\mu_i$. Note that $w$ is the optimal predictor on environment $i$ when using the feature extractor above, hence it automatically zero gradient penalty on environment $i$. By setting $\epsilon=2,\sigma_e=1$ in theorem D.3~\citep{Rosenfeld2020TheRO}, the IRMv1 penalty term environment other than $i$ is upper bounded by $$\sum_{P\in \Omega} \parallel \nabla_w \E_{P}[\ell(w \circ \Phi)]\parallel^2\le O\left(\exp{(-\frac{d_e}{4})(2d_e\exp{(4)}+\bar{\mu}})\right)$$
where $\bar{\mu} = \frac{1}{E}\sum_{k=1}^E \parallel \mu_k \parallel_2^2$. Thus when $d_e >> 1$, the penalty term shrinks rapidly towards 0. 

For $z_e\in \gB_r$, the classifier is the invariant classifier that only uses $z_c$ for prediction, and thus has small ERM loss: $\sum_{P\in \Omega}\E_{P}[\ell(w \circ \Phi)\mathbb{I}(z_e\in \gB_r)]=\Pr(z_e\in \gB_r)E\E_{\gN(\mu_c, \sigma_cI_{d_e})}[\ell (2\frac{\mu_c^Tz_c}{\sigma_c^2})]=\Pr(z_e\in \gB_r)E\min\limits_{w^*} \sum_{P\in \Omega}\E_{P}[\ell(w^* \circ \Phi)]$. For $z_e\in \gB_r^c$, the incurred loss can be upper bounded by 
\begin{align*}
&\sum_{P\in \Omega}\E_{P}[\ell(w \circ \Phi)\mathbb{I}(z_e\in \gB_r^c)]\\
&\le \Pr(z_e\in \gB_r^c)\max_{k\not=i}E\E_{P_k}[\ell (2\frac{\mu_c^Tz_c}{\sigma_c^2} + 2\mu_i^Tz_e)|z_e\in \gB_r^c]\\
&\le \Pr(z_e\in B_r^c(\mu_j))\max_{k\not=i}E\E_{P_k}[\ell (2\frac{\mu_c^Tz_c}{\sigma_c^2} + 2\mu_i^Tz_e)|z_e\in \gB_r^c]\qquad \text{($\gB_r^c \subseteq B_r^c(\mu_j)$)}\\
    &\le \exp{(-\frac{d_e}{8})} \max_{k\not=i}E\E_{P_i}[\ell (2\frac{\mu_c^Tz_c}{\sigma_c^2} + 2\mu_i^Tz_e)|z_e\in \gB_r^c] \qquad \text{(sub-exponential tail bound)}\\
    &= \exp{(-\frac{d_e}{8})}(1+\ln 2+ \max_{k\not=i}E\E_{P_k}[\max(0, -(2\frac{\mu_c^Tz_c}{\sigma_c^2} + 2\mu_i^Tz_e))|z_e\in \gB_r^c])\qquad \text{($\ell(x) \le 1+\ln 2 +\max(0,-x)$)}\\
    &\le \exp{(-\frac{d_e}{8})}(1+\ln 2+ \max_{k\not=i}E\E_{P_k}\left[\left|2\frac{\mu_c^Tz_c}{\sigma_c^2} + 2\mu_i^Tz_e\right||z_e\in \gB_r^c\right])\\
    &\le \exp{(-\frac{d_e}{8})}(1+\ln 2+ 2\max_{k\not=i}E\E_{P_k}\left[\left|\frac{\mu_c^Tz_c}{\sigma_c^2}\right|\right]+\E_{P_k}\left[\left|\mu_i^Tz_e\right||z_e\in \gB_r^c\right])\\
    & = \exp{(-\frac{d_e}{8})}(1+\ln 2+ 2\max_{k\not=i}E\frac{2\parallel \mu_c\parallel^2}{\sqrt{2\pi}}\exp(-\frac{\parallel \mu_c\parallel^2}{2\sigma_c^4})\\
    &\qquad +\frac{\parallel \mu_c\parallel^2}{\sigma_c^2}(1-2\Phi(-\frac{\parallel \mu_c\parallel}{\sigma_c^2}))+\E_{P_k}\left[\left|\mu_i^Tz_e\right||z_e\in \gB_r^c\right]) \numberthis \label{eq:abs-gaussian}\\
    & = \exp{(-\frac{d_e}{8})}(1+\ln 2+ 2E\frac{2d_e}{\sqrt{2\pi}}\exp(-\frac{d_e}{2\sigma_c^4})+\frac{d_e}{\sigma_c^2}+\max_{k\not=i}E\E_{P_k}\left[\left|\mu_i^Tz_e\right||z_e\in \gB_r^c\right])\\
    & = \exp{(-\frac{d_e}{8})}(1+\ln 2+ 2E\frac{2d_e}{\sqrt{2\pi}}\exp(-\frac{d_e}{2\sigma_c^4})+\frac{d_e}{\sigma_c^2}+\max_{k\not=i}E\E_{P_k}\left[\left|\mu_i^Tz_e\right||z_e\in B_r^c(\mu_k)\right])\qquad \text{($\gB_r^c \subseteq B_r^c(\mu_k)$)}\\
    & = \exp{(-\frac{d_e}{8})}(1+\ln 2+ 2E\frac{2d_e}{\sqrt{2\pi}}\exp(-\frac{d_e}{2\sigma_c^4})+\frac{d_e}{\sigma_c^2}+\max_{k\not=i}E\E_{\gN(\mu_k,I_{d_e})}\left[\left|\mu_i^Tz_e\right||\parallel z_e-\mu_k \parallel_2\ge r\right])\\
   & = \exp{(-\frac{d_e}{8})}(1+\ln 2+ 2E\frac{2d_e}{\sqrt{2\pi}}\exp(-\frac{d_e}{2\sigma_c^4})+\frac{d_e}{\sigma_c^2}+\max_{k\not=i}E\E_{\gN(0,I_{d_e})}\left[\left|\mu_i^T(z+\mu_k)\right||\parallel z \parallel_2\ge r\right])\\
   &\le \exp{(-\frac{d_e}{8})}(1+\ln 2+ 2E\frac{2d_e}{\sqrt{2\pi}}\exp(-\frac{d_e}{2\sigma_c^4})+\frac{d_e}{\sigma_c^2}+64d_e\\
   &\qquad+E\E_{\gN(0,\parallel\mu_i\parallel_2^2)}\left[\left|z\right||| z |\ge \parallel\mu_i\parallel_2r\right]) \qquad \text{($|\mu_i\mu_k|\le 64d_e$)} \\
  &\le \exp{(-\frac{d_e}{8})}(1+\ln 2+ 2E\frac{2d_e}{\sqrt{2\pi}}\exp(-\frac{d_e}{2\sigma_c^4})+\frac{d_e}{\sigma_c^2}+64d_e\\
  &\qquad +E\E_{\gN(0,\parallel\mu_i\parallel_2^2)}\left[z|z\ge \parallel\mu_i\parallel_2r\right]) \qquad \text{(Symmetry of Gaussian)}\\
&=\exp{(-\frac{d_e}{8})}(1+\ln 2+ 2E\frac{2d_e}{\sqrt{2\pi}}\exp(-\frac{d_e}{2\sigma_c^4})+\frac{d_e}{\sigma_c^2}+64d_e \\
&\qquad+E\parallel\mu_i\parallel_2\frac{\phi(\parallel\mu_i\parallel_2r)}{\Phi(-\parallel\mu_i\parallel_2r)})\qquad \text{(Conditional tail expectation of Gaussian)}\\
&\le \exp{(-\frac{d_e}{8})}(1+\ln 2+ 2E\frac{2d_e}{\sqrt{2\pi}}\exp(-\frac{d_e}{2\sigma_c^4})+\frac{d_e}{\sigma_c^2}+64d_e+E8\sqrt{d_e}(\frac{\parallel\mu_i\parallel_2r}{\sqrt{2}}+\sqrt{(\frac{\parallel\mu_i\parallel_2r}{2})^2+2}))\numberthis \label{eq:erfc}\\
&\le \exp{(-\frac{d_e}{8})}(1+\ln 2+ 2E\frac{2d_e}{\sqrt{2\pi}}\exp(-\frac{d_e}{2\sigma_c^4})+\frac{d_e}{\sigma_c^2}+64d_e+E8\sqrt{d_e}(\sqrt{d_e}+\sqrt{(d_e+2}))\numberthis \label{tr-upper}
\end{align*}
where $\phi,\Phi$ are the PDF and CDF of standard Gaussian. When $d_e>>1$, the upper bound \Eqref{tr-upper} is approximately zero. \Eqref{eq:abs-gaussian} holds by $\E_{\gN(\mu,\sigma)}[|x|]=2\sigma\phi(\frac{-\mu}{\sigma})+\mu(1-2\Phi(\frac{-\mu}{\sigma}))$. \Eqref{eq:erfc} holds by the following lower bound of Gaussian CDF $\Phi$: $\Phi(-x)\ge \frac{1}{2\sqrt{2}}\textrm{erfc}(\frac{x}{\sqrt{2}})=\frac{\exp(-x^2/2)}{\sqrt{2\pi}(\frac{x}{\sqrt{2}}+\sqrt{(\frac{x}{\sqrt{2}})^2+2})}$ \citep{erfc}. Together, the classifier has smalle ERM loss and IRMv1 penalty, hence it achieves nearly optimal IRMv1 loss when $d_e$ is large.

On the other hand, consider the transfer risk of the constructed classifier. The environmental optimal classifier on the environment $j$ is  $w_j = 
(\frac{2\mu_c}{\sigma_c^2} , {2\mu_j})^T
$ by the construction of $\Phi$. Since $\parallel \mu_i-\mu_k\parallel_2\ge 2\sqrt{2d_e}=2r, \forall k \in \{1,2,\dots E\}-\{i\}$, we know that when $z_e \in B_r(\mu_i),  $,$\Phi = \left.
\begin{aligned}
\begin{bmatrix}
z_c \\ z_e
\end{bmatrix}\end{aligned}\right.$, The incurred transfer risk is lower bounded by applying the predictor $w_j$ on environment $i$:
\begin{align*}
    &\max_{k \in \{1,\dots,E\}-\{i\}}\E_{P_i}\left[\ell(w_k\circ \Phi(x))\right] \\
    &=\max_{k \in \{1,\dots,E\}-\{i\}}\E_{P_i}\left[\ell(w_k\circ \Phi(x))\mathbb{I}(z_e \in B_r(\mu_i))]+\E_{P_i}[\ell(w_k\circ \Phi(x))\mathbb{I}(z_e \in B_r^c(\mu_i))\right]\\
    &=\max_{k \in \{1,\dots,E\}-\{i\}}\E_{P_i}\left[\ell (2\frac{\mu_c^Tz_c}{\sigma_c^2} + 2\mu_kz_e)\right] - \E_{P_i}\left[\left(\ell (2\frac{\mu_c^Tz_c}{\sigma_c^2} + 2\mu_kz_e)-\ell(w_k\circ \Phi(x))\right)\mathbb{I}(z_e \in \gB_r^c(\mu_i))\right]\\
    &\ge\max_{k \in \{1,\dots,E\}-\{i\}} \E_{P_i}\left[\ell (2\frac{\mu_c^Tz_c}{\sigma_c^2} + 2\mu_kz_e)\right] - \Pr(z_e \in B_r^c(\mu_i)) \E_{P_i}\left[\left(\ell (2\frac{\mu_c^Tz_c}{\sigma_c^2} + 2\mu_kz_e)\right)|z_e \in \gB_r^c(\mu_i)\right]\\
     &\ge \E_{P_i}\left[\ell (2\frac{\mu_c^Tz_c}{\sigma_c^2} + 2\mu_jz_e)\right]-\Pr(z_e \in B_r^c(\mu_i)) \E_{P_i}\left[\left(\ell (2\frac{\mu_c^Tz_c}{\sigma_c^2} + 2\mu_jz_e)\right)|z_e \in \gB_r^c(\mu_i)\right] \\
    &\ge   \ell (2\frac{\mu_c^T\mu_c}{\sigma_c^2} + 2\mu_j\mu_i)- \Pr(z_e \in B_r^c(\mu_i)) \E_{P_i}\left[\left(\ell (2\frac{\mu_c^Tz_c}{\sigma_c^2} + 2\mu_kz_e)\right)|z_e \in \gB_r^c(\mu_i)\right]\qquad \text{(Jensen's inequality)}\\
    &\ge  \ell (0)-\Pr(z_e \in B_r^c(\mu_i)) \E_{P_i}\left[\left(\ell (2\frac{\mu_c^Tz_c}{\sigma_c^2} + 2\mu_kz_e)\right)|z_e \in \gB_r^c(\mu_i)\right] \qquad \text{(Decreasing of $\ell$ and $\exists j, \mu_i^T\mu_j \le -\frac{\mu_c^T\mu_c}{\sigma_c^2}$)}\\
    &\ge \ln 2 -\exp{(-\frac{d_e}{8})}(1+\ln 2+ 2\frac{2d_e}{\sqrt{2\pi}}\exp(-\frac{d_e}{2\sigma_c^4})+\frac{d_e}{\sigma_c^2}+64d_e+8\sqrt{d_e}(\sqrt{d_e}+\sqrt{(d_e+2}))\quad \text{(Same analysis as in \Eqref{tr-upper})}
\end{align*}
Hence when $d_e$ is large, the transfer risk of the constructed classifier is closed to $\ln 2$. In addition, by theorem D.3~\citep{Rosenfeld2020TheRO}, for some $\delta>0$, the classifier behaves like the optimal ERM classifier in environment $i$ on $1-\frac{2E}{\sqrt{\pi}\delta}\exp{(-\delta^2)}$ of the test distribution.
\end{proof}

\subsection{Analysis in the 2-d case}

\label{app:toy}

We denote the mean of non-causal means as $\E_i \mu_i = {\sum_{i=1}^E \mu_i}/{E}$.
We consider the setting where the feature extractor is linear, \ie, $\Phi(x)=az_c+bz_e$, and the loss function is logistic loss $f(x)=\log{(1+e^{-x})}$. 

For simplicity, we use the ``sum-sum" version of TRM:
\begin{align*}
    \gR_{TRM}(\Phi) =& \E_{i\not=j}\E_{(x,y)\sim P_j}\left[f(y\Phi(x)w(P_i))\right]
\end{align*}
The minimal values of the objectives are scale-invariant to $a^2+b^2$. W.l.o.g we assume $a^2+b^2=1$ in the following. We show that TRM would learn a robust linear feature extractor when $\E_{i}[\mu_i]=0$, as shown below.
\begin{restatable}{proposition}{prefect}
\label{prop:stationary}
Assume $a^2+b^2=1$ and $\E_{i}[\mu_i]=0$, then the minimizer of the $\gR_{TRM}$ $(\pm 1,0)$. 
\end{restatable}
\begin{proof}
Given feature extractor with parameter $a,b$, the optimal predictor $w(P_i)$ has the closed form $w(P_i) = \frac{2(a\mu_c+b\mu_i)}{a^2+b^2}=2(a\mu_c+b\mu_i)$ in $\gR_{TRM}$.
Since the logistic loss $f(x) = \log{(1+e^{-x})}$ is convex, by Jensen's inequality we have $\forall \Phi$:
\begin{align*}
    &\gR_{TRM}(\Phi) =\E_{i\not=j}\E_{(x,y)\sim P_j}\left[f(y\Phi(x)w(P_i))\right]\\
    &=2\E_{\mu_i,\mu_j,i\not=j}\E_{z_c\sim \gN(\mu_c,1),z_e\sim \gN(\mu_j,1)} f(\frac{2(az_c+bz_e)(a\mu_c+b\mu_i)}{(a^2 + b^2)})\qquad\textrm{(Symmetry of $y=\pm 1$)}\\
    &\ge 2\E_{z_c\sim \gN(\mu_c,1)}f(2\E_{\mu_i,\mu_j}\E_{z_e\sim \gN(\mu_j,1)}[2(az_c+bz_e)(a\mu_c+b\mu_i)])\qquad\textrm{(Jensen's inequality)} \\
    &= 2\E_{z_c\sim \gN(\mu_c,1)}f(2[a^2z_c\mu_c + b\E_j\mu_j(az_c+b \E_i\mu_i) +ab\mu_c\E_i\mu_i])\\
    &= 2\E_{z_c\sim \gN(\mu_c,1)}f(2a^2z_c\mu_c)\qquad\textrm{($\E_i[\mu_i]=0$)}\\
    &= \gR_{TRM}((\pm 1,0))
\end{align*}
\end{proof}

\subsection{Proof of Proposition \ref{prop:implicit}}
\implicit*
\begin{proof}
The function $w(Q;\Phi)$ is implicitly defined by the optimality condition:
\begin{align*}
    \frac{\partial  \E_{Q}[\ell(w(Q) \circ \Phi)]}{\partial w(Q)} = 0
\end{align*}
By the implicit function theorem we know that:
\begin{align*}
    \frac{d w(Q)}{d \Phi} = -H_{w(Q)}^{-1}\frac{\partial^2  \E_{Q}[\ell(w(Q) \circ \Phi)]}{\partial w(Q)\partial \Phi}
\end{align*}
Thus we have
\begin{align*}
    (\frac{\partial \E_{P}[\ell(w(Q) \circ \Phi)]}{\partial w})^T \frac{d w(Q)}{d \Phi}&=-(\frac{\partial \E_{P}[\ell(w(Q) \circ \Phi)]}{\partial w})^TH_{w(Q)}^{-1}\frac{\partial^2  \E_{Q}[\ell(w(Q) \circ \Phi)]}{\partial w(Q)\partial \Phi} \\
    &= \frac{\partial (v_Q^T\frac{\partial \E_{Q}[\ell(w(Q) \circ \Phi )]}{\partial w})}{\partial \Phi}
\end{align*}
where $v_Q=-(\frac{\partial \E_{P}[\ell(w(Q) \circ \Phi )]}{\partial w})^T H_{w(Q)}^{-1}$.
\end{proof}

\section{Experimental Details}
\label{app:exp}

\subsection{Synthetic Datasets}
\label{app:dataset}
For 10C-CMNIST and SceneCOCO, we construct 2 training environments, 1 test environments and 1 validation. Below we discuss the detailed data generation process. All the data are constructed on the publicly available dataset and do not contain personally identifiable information or offensive content. The dataset is splitted evenly across environments. For 10C-CMNIST/SceneCOCO, each environment has 16000/2600 datapoints and validation set has 12000/2200 datapoints. We repeat all the experiments ten times on one GeForce RTX 2020 GPU. 

\subsubsection{10C-CMNIST}

10C-CMNIST dataset consists of 60000 datapoints of dimension (3,28,28) from 10 classes. We use 16000 datapoints for each train/test environment. The remaining 12000 datapoints are used for validation.

We use the digits in the public MNIST dataset~\citep{LeCun2005TheMD} as the invariant~(causal) feature. We use the following ten RGB colors as digit colors~(non-causal feature): (0, 100, 0), (188, 143, 143), (255, 0, 0), (255, 215, 0), (0, 255, 0), (65, 105, 225), (0, 225, 225), (0, 0, 255), (255, 20, 147), (180, 180, 180). For background color~(dummy feature), we randomly pick five RGB colors for each each environment in different runs.
 
For combined shift and label-correlated shift, the two training environments are constructed with $100\%$/$r$ label-digit color correlation, where $r$ is the biased degree. For the two training environments, every digit label is randomly correlated with a digit color with $100\%$/$r$ correlation in different runs. We uniformly sample a digit color for every image for the test environment, indicating no correlation between label and git color.

For each image, the background color is uniformly drawn from five RGB colors. Note that the five background colors for each environment are picked separately in combined shift and label-uncorrelated shift. Hence the background colors across environments can be non-overlapped.

\subsubsection{SceneCOCO}

SceneCOCO dataset consists of 10000 datapoints of dimension (3,64,64) from 10 classes. We use 2600 datapoints for each train/test environment. The remaining 2200 datapoints are used for validation.

We use the following ten objects in the public COCO dataset~\citep{Lin2014MicrosoftCC} as the invariant feature: boat, airplane, truck, dog, zebra, horse, bird, train, bus, motorcycle. For the background scenes~(non-causal feature), we use the following 19 scenes in the public Places dataset~\citep{Zhou2018PlacesA1}: beach, canyon, building facade, staircase, desert~(sand), crevasse, bamboo forest, broadleaf forest, ball pit, kasbah, lighthouse, pagoda, rock arch, oast house, orchard, viaduct, water tower, waterfall, zen garden. For object color~(dummy feature), we randomly pick ten RGB colors for each run.

We randomly pick ten scenes as the non-causal feature for different runs for combined shift and label-correlated shift. Every object is randomly correlated with a background scene with $100\%$/$r$ correlation in different runs.

For each image, the object color is uniformly drawn from five RGB colors. Note that the ten object colors for each environment are picked separately in combined shift and label-uncorrelated shift. The object colors across environments can be non-overlapped.

\subsection{Hyper-parameter selection}
\label{app:hyper}
Below we list the method specific hyper-parameters:

\paragraph{IRMv1~\citep{Arjovsky2019InvariantRM}} The objective of IRMv1 is: $
    \sum_{P\in \Omega}\E_{P}[\ell(w \circ \Phi)] + \lambda \parallel \nabla_w \E_{P}[\ell(w \circ \Phi)]\parallel^2
$. The coefficient of the gradient penality term $\lambda$ is searched over a range of $\{1e-3,1e-2,1e-1,1,10\}$. The number of epochs over which to plug in the gradient penalty term is searched over 1$\sim$5 epochs for all datasets.

\paragraph{VREx~\citep{Krueger2020OutofDistributionGV}:} The objective of VREx is: $
   \sum_{P\in \Omega}\E_{P}[\ell(w \circ \Phi)] + \lambda \textrm{Var}(\{\E_{P}[\ell(w \circ \Phi)]\}_{P\in \Omega})
$. The coefficient of the regularization term (the variance of loss across environments) is searched over a range of $\{1e-3,1e-2,1e-1,1,10\}$. The number of epochs over which to plug in the regularization is searched over 1$\sim$5 epochs for all datasets. 

\paragraph{GroupDRO~\citep{Sagawa2019DistributionallyRN}:} GroupDRO proposes an online algorithm for group distributionally robust optimization. The learning rate of the online exponential gradient descent over the simplex is search over $\{1e-3,1e-2,1e-1\}$.

\paragraph{MLDG~\citep{Li2017DeeperBA}:} MLDG proposes meta learning algorithms for domain generalization. The coefficient of the meta learning objective is search over $\{1,0.5,0.1,0.05\}$. 

\paragraph{Fish~\citep{Shi2021GradientMF}:}Fish uses an inner-loop and outer-loop optimization, which is equivalent to the gradient matching objective. The learning rate in the outer-loop~(meta learning step) is searched over $\{1,0.5,0.1,0.05\}$. 

\paragraph{TRM:} The coefficient of the weighted gradient-matching term is searched over $\{1e-3,1e-2,1e-1,1\}$. We unroll the Taylor expansion for 10 steps in the inverse Hessian vector product approximation, \ie $H^{-1}v \approx H_{10}^{-1}v = \sum_{i=0}^{10}(I-H)^iv$.

We evaluate all the methods both on the train-domain validation (Section~\ref{sec:sce}) set and test-domain validation set (Appendix~\ref{app:extra-exp}).
\subsection{Training}
\subsubsection{Convergence of Algorithm~\ref{alg:trm}}

We first introduce the result from Eq.~3.23 in~\citet{Nemirovski2009RobustSA} and proposition 2 in~\citet{Sagawa2019DistributionallyRN}. Denote the $E-1$-dimension simplex as $\Delta_{E-1}$, and the parameter space of $\Phi$ as $\Theta$. Consider the min-max optimization problem
\begin{align*}
    \max_{\alpha \in \Delta_{E-1}}\min_{\Phi\in \Theta} \sum_{i=1}^{E-1} \alpha_i f_i(\Phi)
\end{align*}
\begin{assumption}
$f_i$ is convex on $\Theta$.
\end{assumption}
\begin{assumption}
We have the unbiased stochastic gradient $\nabla f_i(\Phi;\epsilon)$ of $f_i$, that is $\E_{\epsilon}[\nabla f_i(\Phi;\epsilon)]=\nabla f_i(\Phi)$.
\end{assumption}

Online mirror descent yielding average iterates over $T$ iterations $\bar{\alpha}^T$ and $\bar{\Phi}^T$, has the following guarantee.
\begin{proposition}(\citet{Nemirovski2009RobustSA}, Eq.~3.23).\label{prop:mirror}
    Suppose that Assumptions 1-2 hold. Then we have
    \begin{align*}
        \E_{\epsilon}\left[\max_{\alpha \in \Delta_{E-1}}\sum_{i=1}^{E-1} \alpha_i f_i(\bar{\Phi}^T) - \min_{\Phi\in \Theta}\sum_{i=1}^{E-1} \bar{\alpha}_i^T f_i(\Phi)\right] \le 2 \sqrt{\frac{10(R_{\Phi}^2M_{\Phi}^2+M_{\alpha}^2\log (E-1))}{T}}
    \end{align*}
    where
    \begin{align*}
        &\E_{\epsilon}\left[\parallel \nabla_{\Phi} \sum_{i=1}^{E-1} \alpha_i f_i(\Phi;\epsilon) \parallel_2^2\right] \le M_{\Phi} \\
         &\E_{\epsilon}\left[\parallel \nabla_{\alpha} \sum_{i=1}^{E-1} \alpha_i f_i(\Phi;\epsilon) \parallel_2^2\right] \le M_{\alpha} \\
         &R_{\Phi}^2 = \max_{\Phi} \parallel \Phi \parallel_2^2 - \min_{\Phi} \parallel \Phi \parallel_2^2
    \end{align*}
    for online mirror descent with 1-strongly convex norm $\parallel \cdot \parallel_2$.
\end{proposition}

Recall that $P(Q)=\sum_{P_i \in \Omega\backslash Q} \mathbf{\alpha}_i(Q)P_i$. The TRM risk in \Eqref{risk:per-env} can be formulate the as a saddle point problem:
\begin{align*}
    &\gR(\Phi,w_{all}; Q)=
    \sum_{i:P_i \in \Omega\backslash Q} \alpha_i(Q)\left(\frac{}{}\E_Q[\ell (w_{all}\circ\Phi)]+ \E_{P_i}[\ell(w(Q) \circ \Phi)]\right.\\&\left.\qquad -\lambda D{\langle \mathtt{sg}(\frac{\partial \E_{P(Q)}[\ell(w(Q) \circ \Phi)]}{\partial w}), \frac{\partial \E_{Q}[\ell(w(Q) \circ \Phi)]}{\partial w}\rangle}\right)
\end{align*}

Correspondingly, let $f_i(\Phi,w_{all})=E_Q[\ell (w_{all}\circ\Phi)]+ \E_{P_i}[\ell(w(Q) \circ \Phi)]-\lambda D{\langle \mathtt{sg}(\frac{\partial \E_{P(Q)}[\ell(w(Q) \circ \Phi)]}{\partial w}), \frac{\partial \E_{Q}[\ell(w(Q) \circ \Phi)]}{\partial w}\rangle}$. Denote the average iterate over $T$ iterations as $\bar{\Phi}^T,\bar{w_{all}}^T,\bar{\alpha}(Q)^T$. We define the average per-environment regret as
\begin{align*}
    r_T(Q) = \max_{\alpha(Q)} \gR(\bar{\Phi}^T, \bar{w}_{all}^T; Q)-\min_{\Phi,w_{all},\alpha=\bar{\alpha}(Q)^T}\gR(\Phi,w_{all}; Q)
\end{align*}

\Algref{alg:trm} can be seen as an instance of online mirror descent for saddle point problem above, with the following assumptions:

\begin{assumption}
$f_i(\Phi, w_{all})$ is convex for $(\Phi, w_{all}) \in \Theta$.
\end{assumption}
\begin{assumption}
We have the unbiased stochastic gradient $\nabla f_i(\Phi, w_{all};\epsilon)$ of $f_i$ in each iteration, that is $\E_{\epsilon}[\nabla f_i(\Phi, w_{all};\epsilon)]=\nabla f_i(\Phi, w_{all})$.
\end{assumption}
\begin{proposition}
Suppose that Assumptions 3-4 hold, and $f_i$ is convex, $C$-Lipschitz continuous, and bounded by $B_{\ell}$. Further assume that $\parallel (\Phi, w_{all}) \parallel_2 \le B_{\Phi, w_{all}}$ for all $(\Phi, w_{all})\in \Theta$ with convex set $\Theta\subseteq \mathbb{R}^d$. Then, the average iterate of Algorithm 1 achieves an expected per-environment regret at the rate
\begin{align*}
     \E_{\epsilon}[r_T] \le 2\sqrt{\frac{10(B_{\Phi}^2C^2+B_{\ell}^2\log (E-1))}{T}}
\end{align*}
\begin{proof}
We compare the correspond terms in Proposition~\ref{prop:mirror}.
    \begin{align*}
        &\E_{\epsilon}\left[\parallel \nabla_{\Phi, w_{all}} \sum_{i=1}^{E-1} \alpha_i f_i(\Phi, w_{all};\epsilon) \parallel_2^2\right] \le C^2 \\
         &\E_{\epsilon}\left[\parallel \nabla_{\alpha} \sum_{i=1}^{E-1} \alpha_i f_i(\Phi;\epsilon) \parallel_2^2\right] \le B_l^2 \\
         &R_{\Phi}^2 = \max_{\Phi, w_{all}} \parallel (\Phi, w_{all}) \parallel_2^2 - \min_{\Phi, w_{all}} \parallel (\Phi, w_{all}) \parallel_2^2 \le B_{\Phi, w_{all}}^2
    \end{align*}
We arrive at the result directly by Proposition~\ref{prop:mirror}.
\end{proof}
\end{proposition} 
\label{app:training}

\label{app:converge}
\subsubsection{Training details}

We adopt a 4-layer CNN as the feature extractor for 10C-CMNIST and the Wide ResNet~\citep{Zagoruyko2016WideRN} for SceneCOCO. We use ResNet18/ResNet50 as the backbone for PACS. The predictor $w$ is a fully connected layer.

We train for 10 epochs with batch size 128 on 10C-CMNIST, 100 epochs with batch size 128 on SceneCOCO, and 50 epochs with batch size 32 on PACS. SGD with $1e-1$ initial learning rate and 0.9 momentum as the optimizer for 10C-CMNIST and SceneCOCO. We decay the learning rate with a constant $0.1$ at $4$-th epoch on 10C-CMNSIT and $40$-th epoch for SceneCOCO. We fine-tune on PACS with constant learning rate $1e-4$.

All model-specific hyper-parameters are picked on test-domain or train-domain validation set (Appendix~\ref{app:hyper}).

\section{Extra experimental results}
\label{app:extra-exp}
\subsection{10C-CMNIST and SceneCOCO}

\subsubsection{Different Validation Configuration}
\Figref{img:correlated_train} reports the test accuracy on 10C-CMNIST and SceneCOCO with a validation set that has the same distribution as the test set. 
\begin{figure*}[htbp]
    \centering
  \subfigure[Label-correlated shift]{\includegraphics[width=0.325\linewidth]{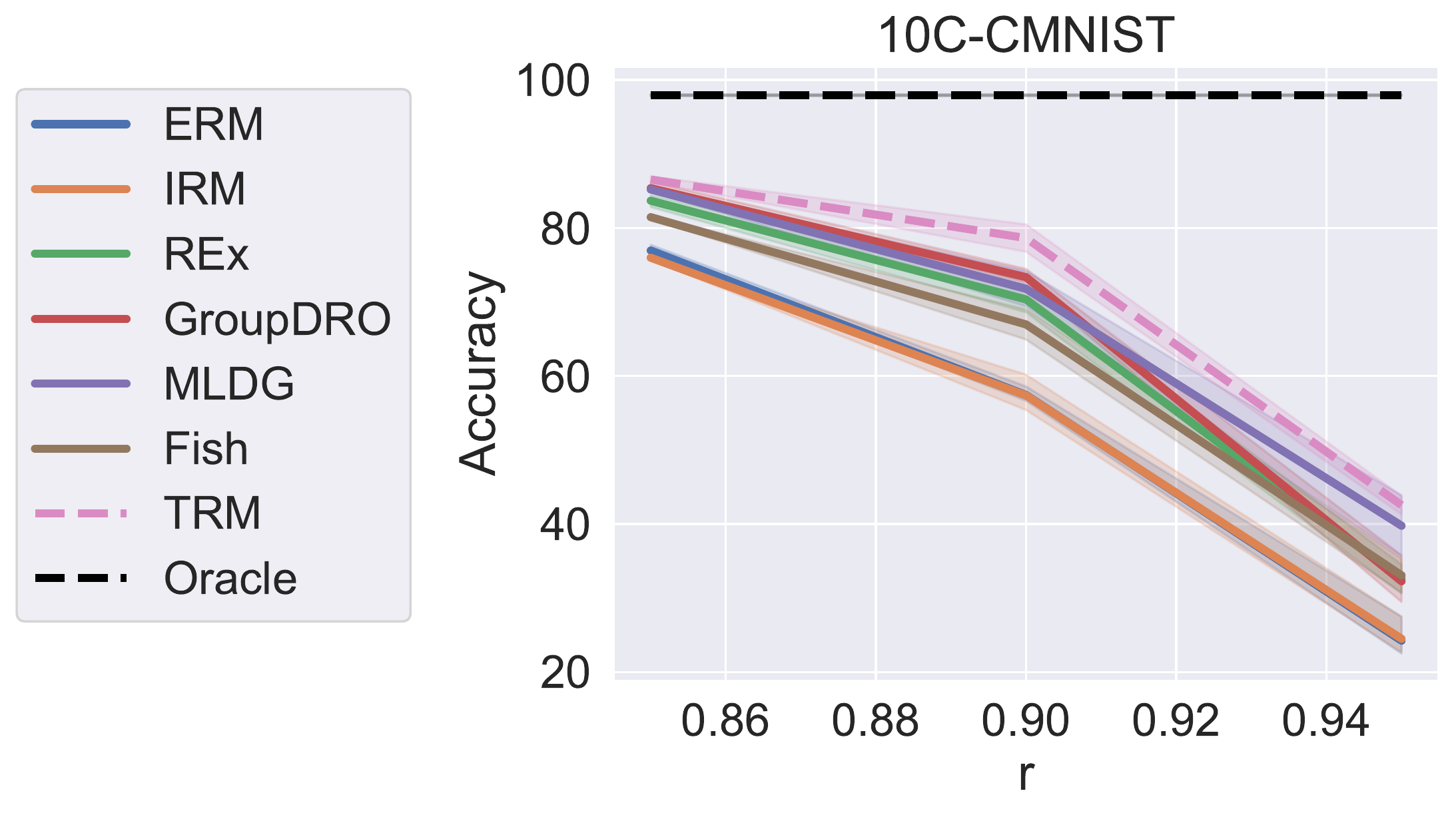}\includegraphics[width=0.22\linewidth]{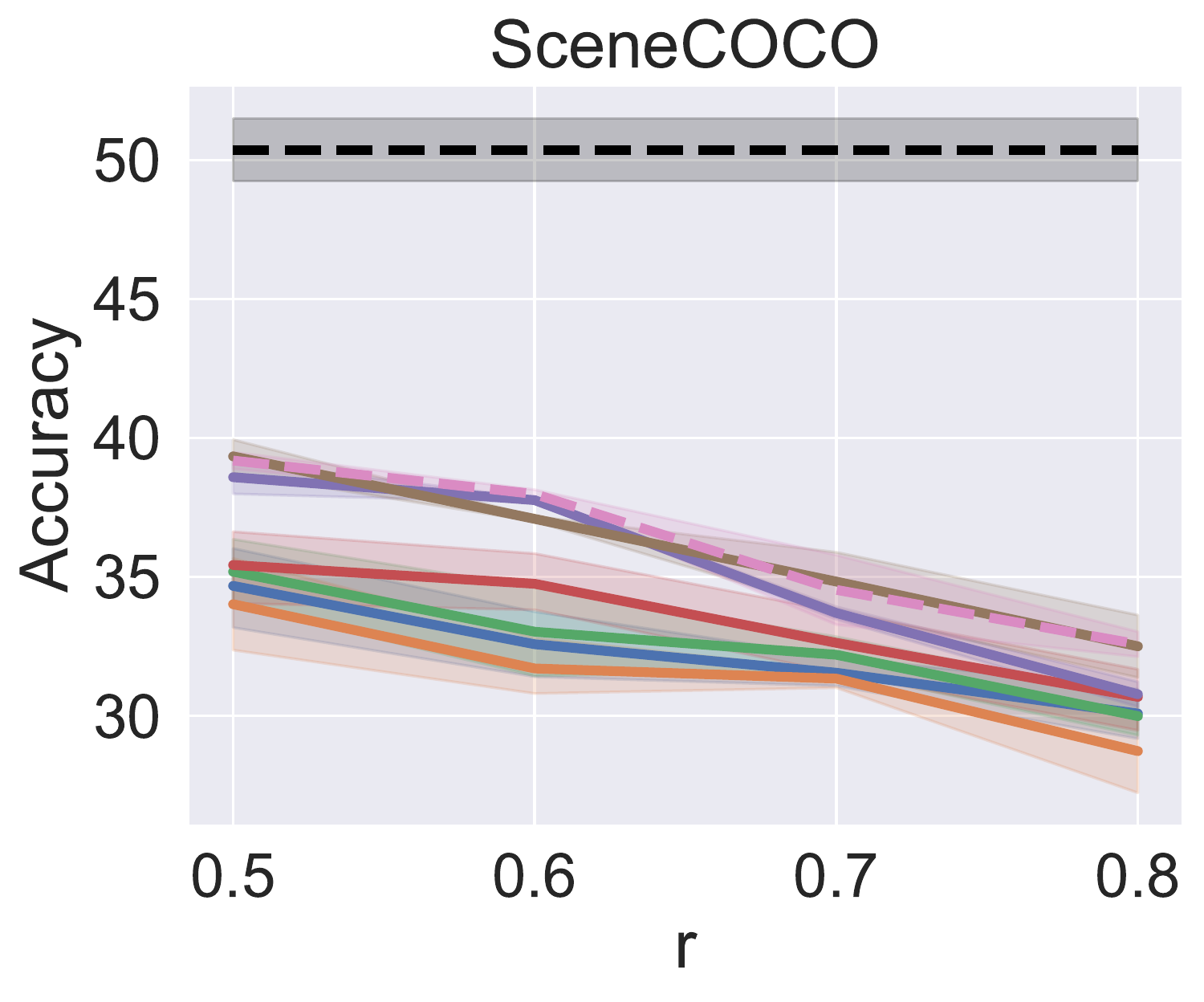}}%
   \subfigure[Combined shift]{\includegraphics[width=0.225\linewidth]{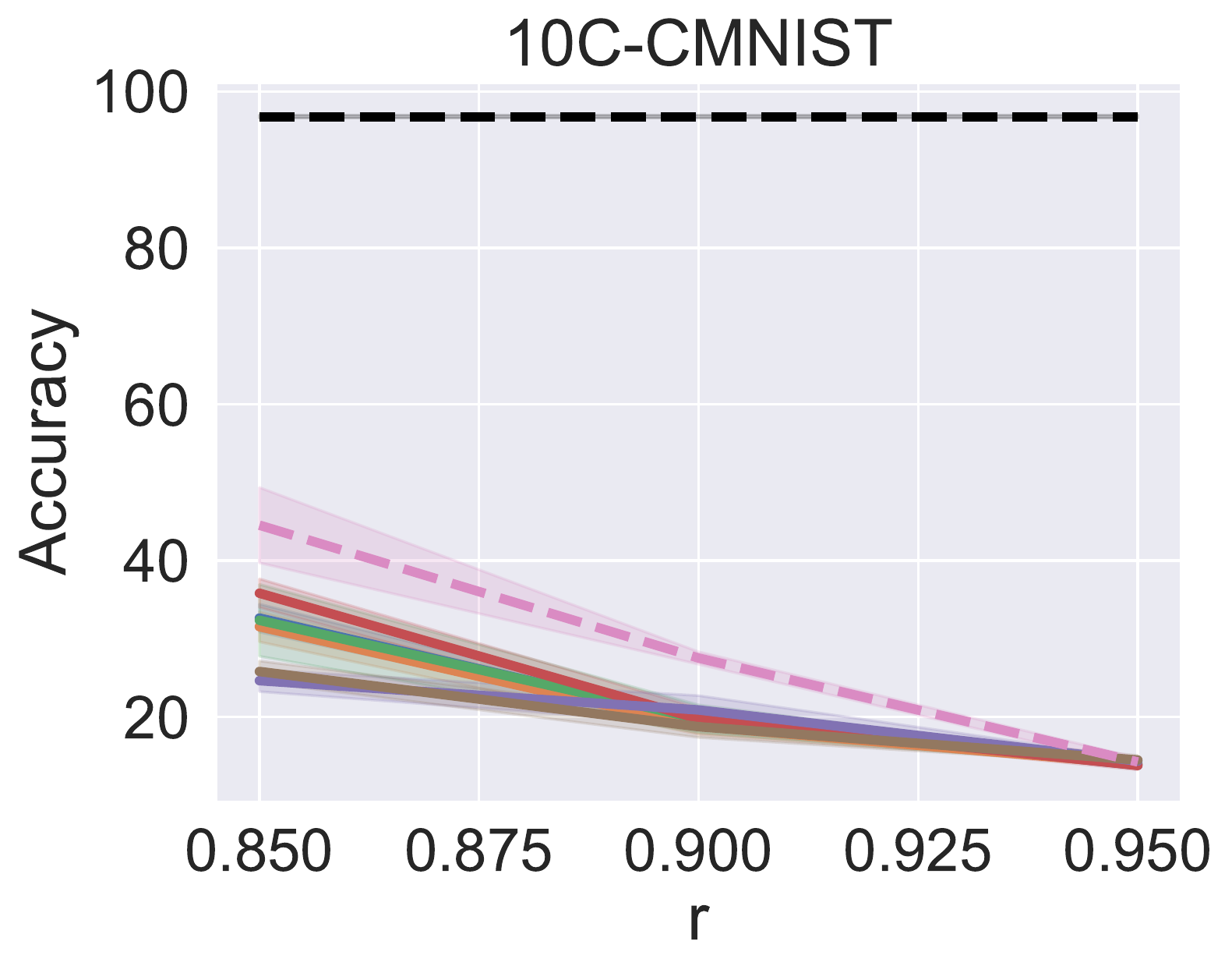}\includegraphics[width=0.22\linewidth]{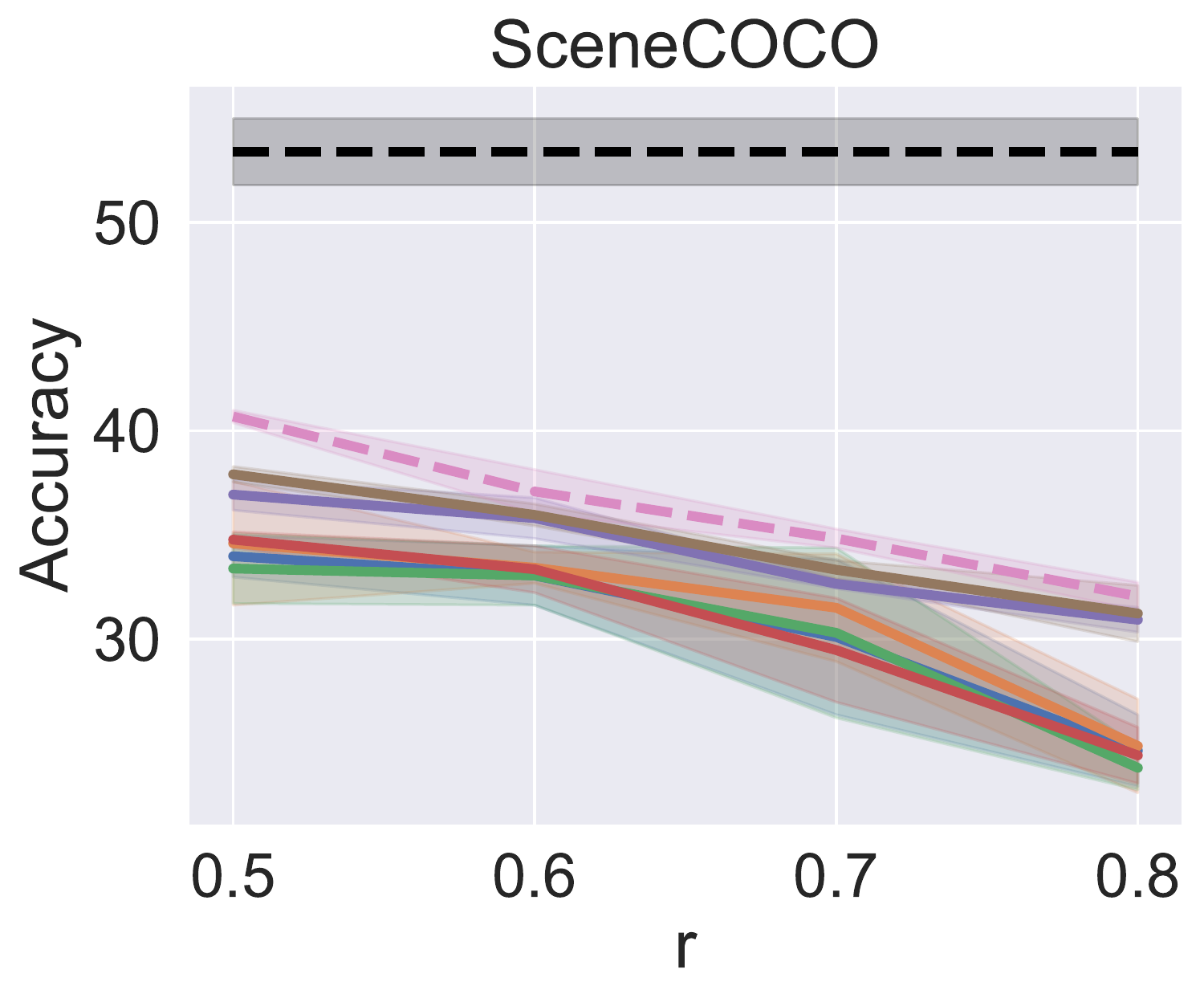}}%
    \caption{Test accuracy on 10C-CMNIST and SceneCOCO datasets in (a) Label-correlated shift and (b) Combined shift, using various bias degrees $r$. The validation set has same distribution as the test set, \ie test-domain validation set.}
    \label{img:correlated_train}
\end{figure*}

\subsubsection{Transferability of models}
\label{app:transfer}

\begin{figure}[htbp]
    \centering
  \subfigure{\includegraphics[width=0.508\linewidth]{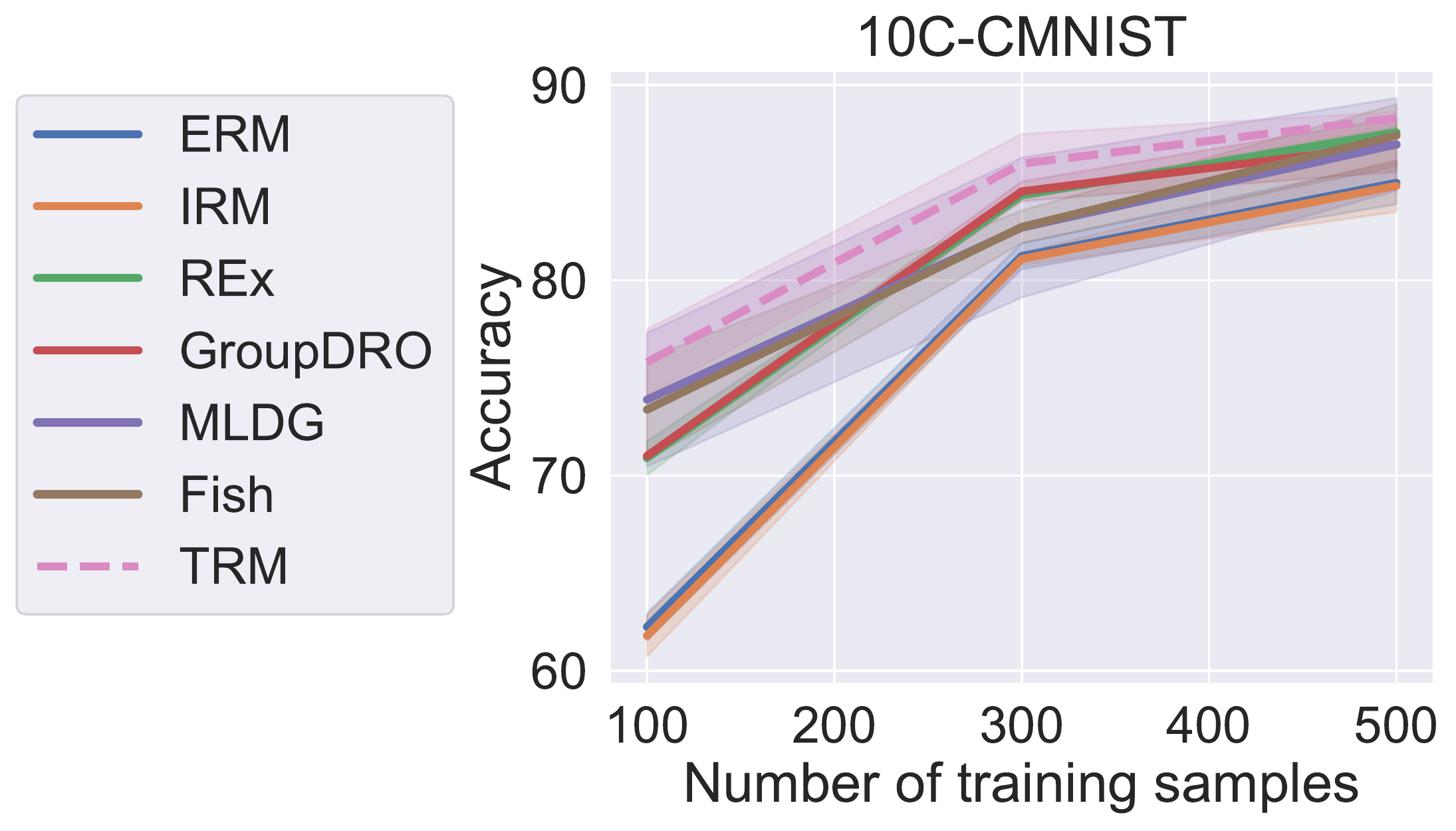}}%
  \subfigure{\includegraphics[width=0.35\linewidth]{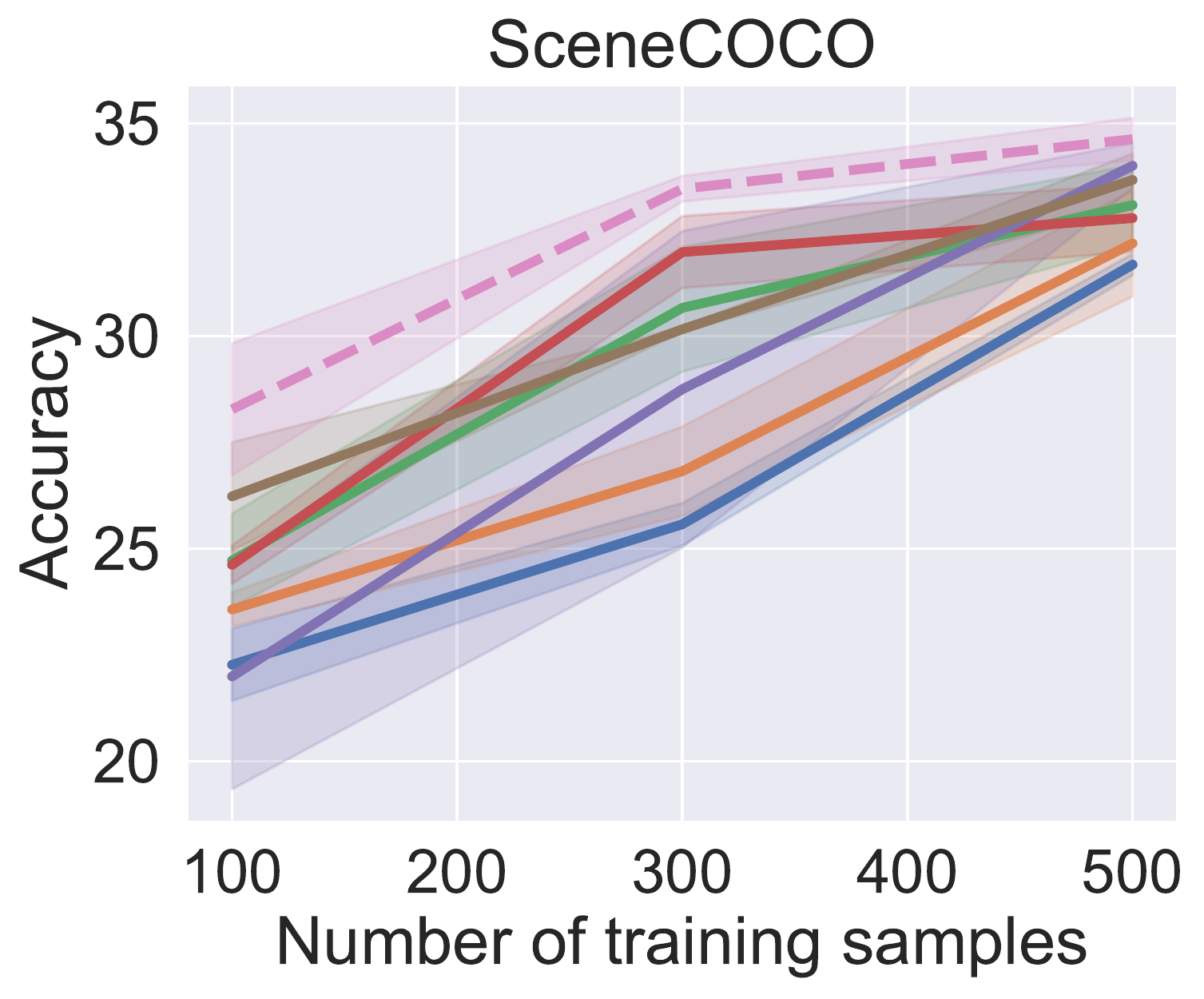}}%
    \caption{Test accuracy on 10C-CMNIST and SceneCOCO datasets in (a) Label-correlated shift and (b) Combined shift, using different numbers of training samples for finu-tuning on the target environment. }
    \label{img:adapt}
\end{figure}
We evaluate the transferability of learned features by fine-tuning the model on the unbiased target environment with limited data. In \Figref{img:adapt}, we report the test accuracy on the target environment after fine-tuning with the different numbers data. We use model trained on $100\%$/$90\%$ and $90\%$/$60\%$ bias degree configurations in 10C-CMNIST and SceneCOCO respectively. Our main finding is that the feature extractors trained by TRM outperform other methods under different numbers of data using for fine-tuning. It indicates that models trained by TRM transfer faster to the target environments. 

\subsubsection{Effect of dummy feature}
\label{app:dummy}
Table~\ref{table:label-uncorr} reports the test accuracy on label-uncorrelated domain shift. We observe that all the methods have similar performance and good generalization in this scenario. The experiment suggests that without the biased effect of non-causal features, learning algorithms are robust for domain shift in general. 
\begin{table*}[htbp]
	    \caption{Test accuracy on Label-uncorrelated shift. All the methods achieve performances comparable to the Oracle.}
    	\label{table:label-uncorr}
		\centering
		\begin{tabular}{c c c c}
		\toprule
		\textbf{Algorithm} &  \textbf{10C-CMNIST} & \textbf{SceneCOCO}\\
		\midrule
        \multirow{1}{*}{ERM} &  $98.0\pm 0.1$ & $64.3\pm 0.9$  \\
        \multirow{1}{*}{IRM} &  $98.0\pm 0.1$ & $65.3\pm 0.7$ \\
        \multirow{1}{*}{REx} &  $98.1 \pm 0.1$ & $65.3\pm 0.7$ \\
        \multirow{1}{*}{GroupDRO} & $98.1 \pm 0.1$ & {66.0 $\pm$ 1.0}\\
        \multirow{1}{*}{MLDG} & $\bf{98.4 \pm 0.1}$ & $\bf{67.4 \pm 0.9}$\\
        \multirow{1}{*}{Fish} & $98.2 \pm 0.1$ & {65.4 $\pm$ 0.7}\\
        \midrule
        TRM & {98.3 $\pm$
        0.1} & $65.5 \pm 1.0$  \\
        \midrule
         Oracle & 98.3 $\pm$
        0.1 & $66.9 \pm 0.6$  \\
		\bottomrule
\end{tabular}
	\end{table*}

\subsection{PACS}

Table~\ref{table:pacs-res518-app} demonstrates the test accuracy on PACS when the validation set has the same distribution as the test set (test-domain validation set), using backbones ResNet18 and ResNet50. Table~\ref{table:pacs-res50-app-train} reports the test accuracy on train-domain validation set, using ResNet50.
\label{app:pacs-exp}
\begin{table*}[htbp]
\begin{center}
\caption{Test accuracy on PACS dataset by ResNet18, on a test-domain validation set.}
\label{table:pacs-res518-app}
\begin{tabular}{c c c c c c c c}
		\toprule
		\textbf{Algorithm} & \textbf{Art} & \textbf{Cartoon} & \textbf{Photo} & \textbf{Sketch}	& \textbf{Average}\\
		\midrule
        \multirow{1}{*}{ERM} & $76.2\pm 1.6$ & $70.6\pm 1.0$ & $95.5 \pm 2.4$ & $66.2\pm 1.1$ & $77.1$ \\
        \multirow{1}{*}{IRM}&$76.0\pm 0.7$ &$70.8\pm 1.0$&$95.2\pm 2.4$& $66.7\pm 0.7$ &$77.2$\\
        \multirow{1}{*}{REx}&$75.6\pm 1.7$&$64.9\pm 1.1$&\textbf{95.8 $\pm$ 0.9}&$63.4\pm 1.0$&$74.9$\\
        \multirow{1}{*}{GroupDRO}& $77.6\pm 0.3$ & $72.9\pm 1.1$& $94.9\pm 2.4$&$69.2\pm 0.9$&$78.6$\\
        \multirow{1}{*}{MLDG}& $77.3\pm 1.2$ & $72.0\pm 0.6$& $95.4\pm 1.0$&$\bf{72.3\pm 0.2}$&$79.2$\\
        \multirow{1}{*}{Fish}& $76.8\pm 1.7$ & $72.9\pm 0.3$& $94.5\pm 1.1$&$70.6\pm 1.0$&$78.7$\\
        \midrule
        TRM & \textbf{78.7 $\pm$ 0.5} &\textbf{74.4 $\pm$ 0.6} & $94.6\pm 0.9$ &$71.4 \pm 2.3$ & \textbf{79.8} \\
		\bottomrule
\end{tabular}
\end{center}
\end{table*}

\begin{table*}[htb]
\begin{center}
\caption{Test accuracy on PACS dataset by ResNet50, on a train-domain validation set.}
\label{table:pacs-res50-app-train}
\begin{tabular}{c c c c c c c c}
		\toprule
		\textbf{Algorithm} & \textbf{Art} & \textbf{Cartoon} & \textbf{Photo} & \textbf{Sketch}	& \textbf{Average}\\
		\midrule
        \multirow{1}{*}{ERM} & $83.7\pm 1.3$ & $68.6\pm 0.4$ & $98.0\pm 0.4$ & $69.8\pm 1.0$ & $80.0$ \\
        \multirow{1}{*}{IRM}&$81.9\pm 1.1$ &$68.6\pm 0.4$&$98.0\pm 0.3$& $70.9\pm 1.0$ &$79.9$\\
        \multirow{1}{*}{REx}&$84.2\pm 1.5$&$68.4\pm 1.0$&$98.0\pm 0.5$&$69.3\pm 0.6$&$80.0$\\
        \multirow{1}{*}{GroupDRO}& $84.3\pm 2.0$& $69.4\pm 0.4$& $98.2 \pm0.3$& $70.8\pm 0.3$&$80.7$\\
        \multirow{1}{*}{MLDG}& $85.0\pm 0.7$& $72.8\pm 0.5$& {$98.0 \pm 0.4$}& $\bf{74.8\pm 0.2}$&$82.7$\\
        \multirow{1}{*}{Fish}& $83.9\pm 0.3$& $71.4\pm 0.4$& \textbf{98.4 $\pm$ 0.1}& $72.3\pm 0.1$&$81.5$\\
        \midrule
        TRM & \textbf{85.8 $\pm$ 1.1} & \textbf{77.3 $\pm$ 0.5}& $97.6\pm 0.2$ & {$70.9\pm 1.2$} & \textbf{82.9} \\
		\bottomrule
\end{tabular}
\end{center}
\end{table*}

\subsection{Training overheads}

In Table~\ref{table:time}, we report the average wall-clock time (seconds/epoch) of different methods on SceneCOCO with batch size 128. We observe that TRM does not take significantly longer training time although it needs to compute the per-environment optimal predictor $w$ and weighted gradient matching term in every minibatch.
\label{app:time}

\begin{table*}[htb]
\begin{center}
\caption{Per-epoch training time of different algorithms.}
\label{table:time}
\begin{tabular}{c c c c c c c c}
		\toprule
		\textbf{Algorithm} &  ERM& IRM & REx & GroupDRO & Fish & MLDG & TRM\\
		\midrule
        \textbf{Wall-clock time (s)} & 5.81 &5.95 & 5.86 & 5.90 &  6.55 & 10.69 & 8.69\\
        \bottomrule
\end{tabular}
\end{center}
\end{table*}


\subsection{Experiments on Group Distributional Robustness}
\label{app:dro}
\subsubsection{Setups}

The group distributional robustness~\citep{Hu2018DoesDR, Sagawa2019DistributionallyRN} aims to minimize the \emph{worse-group accuracy} over a set of pre-defined groups, instead of the average accuracy over these groups. Following \citet{Sagawa2019DistributionallyRN}, we uses \textbf{CelebA} dataset~\citep{Liu2015DeepLF} for evaluation. The hair color \{blond, non-blond\} is used as label and the gender \{male, female\} as the spurious feature. There are four groups in the dataset, namely blond-haired male, blond-haired female, non-blond-haired hair male, non-blond-haired female, with a total of 162700 datapoints and 1387 datapoints in the smallest group (blond-haired male). 

We compare TRM with vanilla \textbf{ERM}, \textbf{GroupDRO}, and \textbf{Reweight}, which sets the sampling weights to the inverse of the group priors~\citep{Sagawa2019DistributionallyRN}. Each group has only one class label and hence hinders the predictor transfer procedure in TRM. Thus we combine blond-haired male and non-blond-haired hair male/blond-haired female and non-blond-haired hair female as two groups for TRM training. We resize all images to $(3,224,224)$ and use the ImageNet-pretrained ResNet18 as the backbone of feature extractors. We use the SGD optimizer with a momentum of $0.9$, a weight decay of 1e-4, and a fixed learning rate of 1e-4. The batch size is set to $32$.
\subsubsection{Results}

In Table~\ref{table:celeba-res18}, we report the average and worse-group test accuracy on CelebA dataset. The proposed TRM algorithm improves over other baselines on the worse-group test accuracy. We observe that TRM still has a competitive average accuracy with other baselines. The results show that TRM more relies on the invariant features~(hair color) for prediction, leading to lower worse-group accuracy, even though the spurious feature~(gender) exists in the dataset.

\begin{table*}[htbp]
\begin{center}
\caption{Worse-group and average accuracy on CelebA dataset}
\label{table:celeba-res18}
\begin{tabular}{c c c c c c c}
		\toprule
		 & ERM & Reweight & GroupDRO & TRM \\
		\midrule
        \multirow{1}{*}{Worse-group accuracy} & $46.0\pm 2.9$ & $89.3\pm 1.2$& $90.0\pm 1.3$ & $\bf{90.3\pm 0.4}$ \\
          \multirow{1}{*}{Average accuracy} & $\bf{94.7\pm 0.3}$ & $91.8\pm 0.4$& $91.6\pm 0.3$ & $91.6\pm 0.4$ \\
        \bottomrule
\end{tabular}
\end{center}
\end{table*}

\end{document}